%% file: main.tex
\documentclass{article}

\usepackage{microtype}
\usepackage{graphicx}
\usepackage{subfigure}
\usepackage{booktabs} 

\input{preamble}

\definecolor{blued}{RGB}{50, 70, 200}

\usepackage{hyperref}

\usepackage[accepted]{icml2024}

\usepackage[capitalize,noabbrev]{cleveref}

\usepackage[textsize=tiny]{todonotes}
\usepackage{enumitem}

\icmltitlerunning{Test-Time Regret Minimization in Meta Reinforcement Learning}

\begin{document}

\twocolumn[
\icmltitle{Test-Time Regret Minimization in Meta Reinforcement Learning}




\begin{icmlauthorlist}
\icmlauthor{Mirco Mutti}{yyy}
\icmlauthor{Aviv Tamar}{yyy}
\end{icmlauthorlist}

\icmlaffiliation{yyy}{Technion -- Israel Institute of Technology, Haifa, Israel}

\icmlcorrespondingauthor{Mirco Mutti}{mirco.m@technion.ac.il}

\icmlkeywords{Machine Learning, ICML}

\vskip 0.3in
]



\printAffiliationsAndNotice{\icmlEqualContribution} 

\begin{abstract}
Meta reinforcement learning sets a distribution over a set of tasks on which the agent can \emph{train} at will, then is asked to learn an optimal policy for any \emph{test} task efficiently. In this paper, we consider a \emph{finite} set of tasks modeled through Markov decision processes with various dynamics.
We assume to have endured a long training phase, from which the set of tasks is perfectly recovered, and we focus on \emph{regret minimization} against the optimal policy in the unknown test task. 
Under a separation condition that states the existence of a state-action pair revealing a task against another, \citet{chen2021understanding} show that $\cO(M^2 \log(H))$ regret can be achieved, where $M, H$ are the number of tasks in the set and test episodes, respectively. In our first contribution, we demonstrate that the latter rate is nearly optimal by developing a novel \emph{lower bound} for test-time regret minimization under separation, showing that a linear dependence with $M$ is unavoidable. Then, we present a family of stronger yet reasonable assumptions beyond separation, which we call \emph{\AssFam}, enabling algorithms achieving fast rates $\log (H)$ and sublinear dependence with $M$ simultaneously. Our paper provides a new understanding of the statistical barriers of test-time regret minimization and when fast rates can be achieved. 
\end{abstract}

\section{Introduction}

\begin{table*}[t!]
\setlength{\tabcolsep}{4pt}
\begin{threeparttable}
\begin{small}
\begin{tabular}{c c c c}
\toprule
\begin{sc} Structure \end{sc} & \begin{sc}Reachability \end{sc}& \begin{sc}Regret Upper Bound  \end{sc}& \begin{sc}Regret Lower Bound\end{sc} \\
\midrule
 Separation (Ass.~\ref{ass:separation_mdp}) & Reachable (Ass.~\ref{ass:reachability})  & $\cO \left( \frac{T M^2 \log (MH)}{\lambda^4} \right)$ (Thr.~\ref{thr:mdp_upper_bound_separation})  & $\Omega \left( \frac{T M \log (H)}{\lambda} \log \left( \frac{1}{\delta} \right) \right)$ (Thr.~\ref{thr:mdp_lower_bound})\textsuperscript{*} \\
\rowcolor{lightgray} Clustering (Ass.~\ref{ass:mdp_clustering}) & Cluster Reachable (Ass.~\ref{ass:mdp_clustering_reachability})  & $\cO \left( \frac{T (K^2 + N^2) \log (N H) }{\lambda^4} \right)$ (Thr.~\ref{thr:mdp_upper_bound_clustering})\textsuperscript{*}  &  \\
 Tree Structure (Ass.~\ref{ass:mdp_tree}) & Strongly Reachable (Ass.~\ref{ass:strong_reachability})  & $\cO \left( \frac{T d \log (dH)}{\lambda^4} \right)$ (Thr.~\ref{thr:mdp_upper_bound_tree})\textsuperscript{*}  & \\
 \rowcolor{lightgray} Revealing Policies (Ass.~\ref{ass:mdp_revealing_policies}) & Reachable (Ass.~\ref{ass:reachability})  & $\cO \left( \frac{T I \log (MH)}{\lambda^4} \right)$ (Thr.~\ref{thr:mdp_upper_bound_revealing_policies})\textsuperscript{*}  &  \\
\bottomrule
\end{tabular}
\end{small}
\end{threeparttable}
\vspace{-7pt}
\caption{Overview of the theoretical results. $T$ is the horizon of an episode, $M$ the number of tasks, $H$ the number of episodes, $\lambda$ a separation parameter, $\delta$ a confidence, $K$ the number of clusters with size $N$, $d = \log_{1/\beta} M$ the depth of a binary tree with entropy factor $\beta$, $I$ the number of revealing policies. Note that the rates are simplified to highlight relevant factors. Novel results are marked with *.  }
\label{table:regret_overview}
\vspace{-7pt}
\end{table*}

Reinforcement Learning~\citep[RL,][]{sutton2018reinforcement} is a popular tool for learning an optimal decision policy through sampled interactions with a Markov Decision Process (MDP), a general framework encompassing countless applications, ranging from robotics~\cite{xu2023dexterous, kaufmann2023champion} to algorithms design~\cite{fawzi2022discovering}, conversational agents~\cite{stiennon2020learning}, and others.

Although powerful, the efficiency of RL is a long-standing issue. The theory says that the \emph{regret} of a RL algorithm, i.e., the difference between the value of the deployed policy and the optimal policy in hindsight, inescapably scales with $\sqrt{H}$ in the worst case~\cite{jaksch2010near, osband2016lower}, $H$ being the total number of episodes of interactions with the MDP.  Even if the real world is arguably better behaved than the worst-case MDP, the most successful algorithms~\citep[][]{schulman2015trust, mnih2015human} still take thousands of interaction episodes to learn a competitive policy in simulation, which draws pessimism for RL to be applied for learning in the real world.

A promising direction to improve RL efficiency is \emph{meta RL}~\cite{ghavamzadeh2015bayesian}, in which a distribution over the set of tasks we can face is considered. In meta RL, we first have a \emph{training} phase on some tasks sampled from the latter distribution, for which the learning efficiency is less of an issue (e.g., a simulator is available). Then, we exploit the collected knowledge to achieve faster learning in a \emph{test} task, which is assumed to come from the same distribution. 

Much of the previous work in meta RL focuses on algorithms for the training stage~\cite{duan2016rl, finn2017model, rakelly2019efficient, zintgraf2019varibad, zintgraf2021varibad}, or analyse generalization of the trained model to the test task~\cite{simchowitz2021bayesian, tamar2022regularization, rimon2022meta, zhao2022effectiveness, zisselman2023explore}.

Here we study meta RL from a different perspective. We assume to have spent infinite time in the training phase, such that the task distribution can be recovered (we mean the full specifications of all the MDPs in the set, not just the task distribution itself), and we aim to minimize the regret against the optimal policy in the test task. Especially,

{\centering
    \textit{ Does perfect meta RL training provably improve the 
    learning efficiency on the test task against standard RL?}
\par}
We believe that a positive answer is an essential theoretical ground for motivating meta RL, as there is little incentive to undergo a costly training (at least in terms of computation) without guarantees of improved efficiency on the test task.

Even in simple settings, in which the distribution is supported on a finite set of $M$ tasks, meta RL provides little hope, as the regret still scales with $\sqrt{H}$ in the worst case, with only marginal gains in the statistical efficiency w.r.t. standard RL~\cite{chen2021understanding, ye2023power}.
Nonetheless, under a common \emph{separation} assumption on set of tasks~\citep{chen2021understanding, kwon2021rl}, i.e., there exists at least one reachable state-action pair that \emph{reveals} one task against the others, the prospects of meta RL become brighter.
\citet{chen2021understanding} show that $\cO(M^2 \log (MH))$ regret can be achieved by first identifying the test task (with high probability) and then deploying the best policy for the latter. Their approach is somewhat wasteful in the identification, as the algorithm performs a sequence of \emph{one-vs-one} tests on candidate tasks, which induces the $M^2$ factor. However, it is unclear if the latter is necessary or better algorithms could be developed.

In this paper, we provide a nuanced study of the statistical barriers of test-time regret minimization in meta RL.
First, we provide a lower bound $\Omega (TM \log(H))$ for test-time regret minimization under separation, where $T$ is the horizon of an interaction episode with the test MDP. Our lower bound demonstrates that the ``wasteful'' algorithm by~\citet{chen2021understanding} is nearly optimal and one factor $M$ is unavoidable under separation alone. The way the lower bound is derived philosophically confirms that probing the MDP to first identify the test task and then exploit the collected information is not just reasonable but also optimal. Nevertheless, we note that a linear dependence with $M$ is not desirable if we aim to scale meta RL to large task distributions.

Thus, we formulate stronger yet reasonable requirements under which the latter limitation may be overcome. To formalize our desiderata, we call \emph{\AssFam} a family of assumptions that allow for fast rates in $H$ and $M$ simultaneously, for which we provide three instances.

The first assumes the tasks can be partitioned in coherent \emph{clusters}, such that the algorithm can efficiently identify the cluster to which the test task belongs, and then identify the latter within that cluster down the hierarchy. Clusters are fairly natural in practice, e.g., in a meta RL problem applied to movie recommendation systems, in which the set of tasks consists of different users with specific tastes. We could first provide recommendations to probe whether the test user especially likes a movie genre, and then provide fine-grained recommendations within the genre catalogue.

The second assumes that the test task can be identified following a \emph{tree structure}. This means that the algorithm can \emph{split} a subset from the current set of candidate tasks by collecting information on a single state-action pair, which allows to eliminate a large chunk of candidates at every iteration. This may happen in practice whenever the test task can be identified testing a sequence of revealing characteristics.

The third {\AssFam} assumption we analyse is admitting the presence of a small number of \emph{revealing policies}, which allow for collecting highly informative data irrespective of the test task, then to perform the identification without further interactions. To provide intuition on such policies, let us think of a meta RL problem in which the set of tasks consists of morphologies of terrains we aim to explore without daylight. A revealing policy would light a torch to probe the surroundings before taking the right direction for the test morphology.

\textbf{Contributions.}~~
Our main contributions are:
\begin{itemize}[topsep=-2pt, noitemsep, leftmargin=*]
    \item We revise an analysis of domain randomization~\citep{chen2021understanding} through the lenses of meta RL, adapting their result on sim-to-real gap into a regret upper bound $\cO(M^2 \log(M H))$ for our setting (Section~\ref{sec:regret_under_separation}); 
    \item We derive a \emph{lower bound} $\Omega (TM \log(H))$ to test-time regret minimization under separation (Section~\ref{sec:lower_bound}) through original techniques that formally link our problem to Best Policy Identification~\citep[BPI,][]{fiechter1994efficient}. The proof requires a \emph{tailored} lower bound to the sample complexity of BPI, which can be of independent interest (Appendix~\ref{apx:best_policy_identification});
    \item We present structural assumptions beyond separation, called \emph{\AssFam} (Section~\ref{sec:beyond_separation}). Those include: A \emph{clustering} assumption that leads to a regret upper bound of order $\cO(T (K^2 + N^2) \log(N H))$, where $K$ is the number of clusters and $N$ is the size of the largest cluster; A \emph{tree structure} assumption that leads to a regret of order $\cO (T d \log (d H) )$, where $d = \log_{1 / \beta} M$ is the depth of the tree and $\beta$ is the splitting factor; A \emph{revealing policies} assumption that leads to a regret of order $\cO (T I \log (M H))$, where $I$ is the number of revealing policies;
    \item We provide additional sharp rates to the test-time regret for meta learning in bandits (Appendix~\ref{apx:bayesian_bandits}).
\end{itemize}
Complete proofs of the theorems are in Appendix~\ref{apx:proofs}.

\section{Problem Formulation}
\label{sec:problem}
We first present the necessary background on MDPs and meta RL (Sections~\ref{sec:problem_mdp},~\ref{sec:problem_meta_rl}) before formulating the learning problem we will address in the paper (Section~\ref{sec:problem_test_time_regret}).

\textbf{Notation.}~~Let $\A$ a space of size $|\A|$ with elements $a \in \A$. Then, $\Prob (\A) := (p \in [0, 1]^{|\A|} | \sum_{a \in \A} p(a) = 1)$ is the simplex for a finite $\A$. Let $p, q \in \Prob (\A)$, their $\ell 1$-distance is $\| p - q  \|_1 = \sum_{a \in \A} |p (a) - q(a)|$, their total variation is $\TV (p, q) = \sup_{a \in \A} |p(a) - q(a)|$, their Kullback-Leibler divergence is $\KL (p, q) = \sum_{a \in \A} p(a) \log (p (a) / q(a))$ and $\KL_{p | q} := \KL (p, q) + \KL (q, p)$. We will denote sets and sequences as $(a_i)_{i \in [I]} := (a_1, \ldots, a_I)$, where $[I] := (1, \ldots, I)$ for some constant $I \in \mathbb{N}$.

\subsection{Markov Decision Processes and RL}
\label{sec:problem_mdp}
A finite-horizon time-homogeneous Markov Decision Process~\citep[MDP,][]{puterman2014markov} is defined by a tuple\footnote{The meaning of the subscripts will become clear later.} $\M_i := (\S, \A, p_i, r_i, s_1, T)$ where $\S$ is a finite set of states ($S = |\S|$), $\A$ is a finite set of actions ($A = |\A|$), $p_i: \S \times \A \to \Prob (\S)$ is a transition model such that $p_i (s' | s, a)$ denotes the conditional probability of transitioning to $s'$ taking action $a$ in state $s$, $r_i: \S \times \A \to [0, 1]$ is a reward function such that $r_i (s, a)$ is the reward collected by taking action $a$ in $s$, $s_1$ is the initial state,\footnote{Note that unique initial state is without loss of generality, as we can accommodate an initial state distribution $\mu \in \Prob(\S)$ through a fictitious state $s_0$ such that $p(s|s_0, a) = \mu (s), \forall a \in \A$.} and $T < \infty$ is the horizon of an episode.

An episode of interaction between an agent and the MDP $\M_i$ goes as follows. At each step $t \in [T]$, the agent observes $s_t \in \S$ and takes $a_t \in \A$. The environment transitions to $s_{t + 1} \sim p_i (\cdot | s_t, a_t)$ and the agent collects $r_i (s_t, a_t)$. 
Hence, an episode can be summarized through a sequence $\tau = (s_t, a_t)_{t \in [T]}$ called a \emph{trajectory}.

The agent selects its actions by means of a non-stationary Markovian \emph{policy} $\pi := (\pi_t: \S \to \Prob (\A))_{t \in [T]} \in \Pi$ where $\pi_t (a | s)$ denotes the conditional probability of action $a$ in state $s$ at time step $t$, and $\Pi$ is the set of policies. We define the \emph{value} at step $t$ of playing policy $\pi$ in state $s$ of $\M_i$ as
\begin{equation*}
    V_{it}^\pi (s) := \EV_{\pi, \M_i} \left[ \sum_{ t' = t}^T r_i (s_{t'}, a_{t'}) \ \Big| \ s_t = s \right],
\end{equation*}
where the expectation is on all the sources of randomness, i.e., the action selection induced by $\pi$ and the state transitions induced by $p_i$, which may be stochastic.
We further denote $V_i (\pi) := V_{i1} (s_1)$ the value of the policy in the initial state. The objective function of the agent in $\M_i$ can then be written as $\max_{\pi \in \Pi} V_i (\pi)$, where we denote as $\pi^*$ the policy attaining the maximum and $V^*_i := V_i (\pi^*)$.
RL~\cite{sutton2018reinforcement} is a paradigm for learning an (approximately) optimal policy $\pi$, such that $V_i^* - V_i (\pi) \leq \epsilon$ for some $\epsilon > 0$, from sampled interactions with an \emph{unknown} MDP $\M_i$.

\subsection{Meta Reinforcement Learning}
\label{sec:problem_meta_rl}
Meta RL~\cite{duan2016rl}, initially introduced by~\citet{schmidhuber1987evolutionary}, extends the RL paradigm to a set of $M$ MDPs $\M := (\M_i)_{i \in [M]} = (\S, \A, p_i, r_i, s_1, T)_{i \in [M]}$ having the same $\S, \A, s_1, T$, but possibly different transition model $p_i$ and reward $r_i$. Just like in RL, the latter MDPs are typically unknown to the agent. In a process called \emph{training}, the agent can collect interactions with a number of tasks\footnote{We are going to use the term \emph{task} and \emph{MDP} interchangeably.} drawn from $\M$ according to a \emph{task distribution} $P \in \Prob (\M)$ such that the probability of drawing $\M_i$ is $P(\M_i)$. In training, the agent collects information into a \emph{prior} model, e.g., a policy, a model of transitions, or an algorithm, that is then used to address a RL problem on a \emph{test} task $\M_i$ assumed to be drawn from the same task distribution $P$.

Bayesian RL~\cite{ghavamzadeh2015bayesian} formulates the target of the training as learning a \emph{Bayes-optimal} policy\footnote{Note that the Bayes-optimal policy is history-dependent in general~\cite{ghavamzadeh2015bayesian}.} $\pibo \in \argmax_{\pi \in \Pi} \EV_{\M_i \sim P} [V_{i} (\pi)]$ under the distribution $P$. As we describe below, here we study meta RL from a \emph{frequentist} perspective rather than a Bayesian formulation.

\subsection{Test-Time Regret Minimization}
\label{sec:problem_test_time_regret}
In this paper, instead of focusing on the training phase of meta RL, we assume perfect knowledge of the set of tasks $\M$, such that every transition model $p_i$ and reward $r_i$ are fully known to the agent.
With this prior knowledge, we aim to minimize the \emph{test-time regret} over $H$ episodes
\begin{equation}
\label{eq:test-time-regret}
    R_H (\M_i, \mathbb{A}) := \EV \left[ \sum_{h = 1}^H V^*_i - V_i (\pi_h) \right]
\end{equation}
where $\M_i \in \M$ is \emph{any} test task, $\pi_h$ is the policy deployed in episode $h$ by algorithm $\mathbb{A}$, and the expectation is over realizations of the episodes $1, \ldots, H$ taken from $\M_i$.  The motivation for this objective is that comparing against the optimal policy \emph{for} the test task, instead of an optimal policy \emph{on average} over the task distribution~\cite{ye2023power}, gives a regret measure that is robust to the worst-case task, arguably a minimal requirement given the perfect training assumption.
We see the latter as a necessary first step towards a more realistic setting with \emph{approximate} knowledge of $\M$ only. If we cannot succeed with the former, the latter is hopeless.

Two other important observations are in order. First, in this paper we study the regret of \emph{adaptive} algorithms $\mathbb{A}$, that is, the deployed policy is Markovian within an episode, but can change from episode to episode. This is only slightly restrictive as the policy $\pi_h$ is computed having the full history of realizations in previous episodes $\Hist_h = ((s_{t,h},a_{t,h},r_{t,h})_{t \in [T]})_{h' \in [h]} $, which means an algorithm $\mathbb{A}$ corresponds to a non-Markovian policy with low switching cost~\cite{bai2019provably}.
Second, the expression in~\eqref{eq:test-time-regret} is different from the notion of \emph{Bayesian regret} that is common Bayesian RL~\citep{ghavamzadeh2015bayesian}, in which the regret is taken in expectation over the task distribution $P$ instead of taking the worst-case task. As a result, the optimal algorithm $\mathbb{A}$ for the test-time regret does not correspond to the Bayes-optimal policy in general, although it holds $R_H (\M_i, \pibo) = \cO( M \cdot R_H(\M_i, \mathbb{A}))$ for any algorithm $\mathbb{A}$ from~\citep[][Lemma 1]{chen2021understanding}.

\section{Previous Fast Rates for Test-Time Regret}
\label{sec:regret_under_separation}
In this section, we discuss the known results for the test-time regret minimization objective we described above. In this paper we especially care for \emph{fast rates}, i.e., those settings in which the knowledge of the set of tasks and its structure allow to overcome the statistical barrier for regret minimization in finite-horizon RL, which we know is of order $\Theta(\poly(T, S, A)\sqrt{H})$ from lower bounds and minimax algorithms~\cite{osband2016lower, azar2017minimax}.

\citet{chen2021understanding} address a regret minimization problem that is very close to our test-time regret formulation, although their narrative is centered around domain randomization rather than meta RL. When the set of tasks is finite, they provide a lower bound of order $\Omega (\sqrt{D M H})$, in which $D$ is the diameter of a communicating infinite-horizon MDP~\citep[see Assumption 1 and Theorem 3 in][]{chen2021understanding}.\footnote{Note that the results in \citet{chen2021understanding} are given for a slightly different setting (detailed comparisons are in Section~\ref{sec:related_works}). We will explicitly adapt to our setting the most relevant results.}
The latter result demonstrates that additional assumptions are needed to break the $\sqrt{H}$ barrier of RL. 

To this end, \citet{chen2021understanding} introduce a separation condition within the set of tasks $\M$. Formally, 
\begin{assumption}[$\lambda$-separation~\citep{chen2021understanding}]
\label{ass:separation_mdp}
    For any $\M_i, \M_j \in \M$, there exists $(s, a) \in \S \times \A$ such that $\| (p_i - p_j)(\cdot | s, a) \|_1 \geq \lambda.$
\end{assumption}
The latter assumption guarantees the existence of a \emph{revealing} state-action pair to tell apart a task from another. This allows to design an algorithm which repeatedly tests that revealing state-action pair to identify the test task efficiently. First, we need to further make sure that the revealing state-action can be reached with meaningful probability.\footnote{This is the technical adaptation of the communicating MDP assumption in~\citet{chen2021understanding} for the finite-horizon setting.}
\begin{definition}
    Let $X (s | \M_i, \pi)$ denote the random variable of the first time step in which the state $s \in \S$ is reached by playing policy $\pi$ in the MDP $\M_i$. Let $X(s,a | \M_i, \pi)$ be the analogous for state-action pairs $(s, a) \in \S \times \A$.
\end{definition}
\begin{assumption}[Reachable MDPs]
\label{ass:reachability}
    An MDP $\M_i$ is \emph{reachable} if it holds $\min_{\pi \in \Pi} \EV [X (s | \M_i, \pi)] \leq T / 2, \forall s \in \S$.
\end{assumption}

With the combination of the latter assumptions, we can directly adapt the algorithmic solution in~\citet[][]{chen2021understanding} to the finite-horizon setting.\footnote{We refer to Algorithm 1 in~\citet{chen2021understanding}, which we name here the ``Identify-then-Commit'' algorithm.} We report the pseudocode of the resulting procedure in Algorithm~\ref{alg:mdp_separation}.

\begin{algorithm}[t]
    \caption{Identify-then-Commit \citep{chen2021understanding}}
    \label{alg:mdp_separation}
\begin{small}
    \begin{algorithmic}[1]
        \STATE \textbf{input} set of MDPs $\D$ and visitation count $n$
        \WHILE{$|\D| > 1$}
            \STATE Draw $\M_1, \M_2$ from $\D$ randomly
            \STATE Let $(\bar s, \bar a) \in \argmax_{(s, a) \in \S \times \A} \| (p_1 - p_2) (\cdot | s, a) \|_1$
            \STATE Call  Algorithm~\ref{alg:mdp_sampling} with $\D, (\bar s, \bar a), n$ to collect $\X$
            \IF{$\exists s' \in \X: p_2 (s' | \bar s, \bar a) = 0$ or $\prod_{s' \in \X} \frac{p_1 (s' | \bar s, \bar a)}{p_2 (s' | \bar s, \bar a)} \geq 1$}
                \STATE Eliminate $\M_2$ from $\D$
            \ELSE
                \STATE Eliminate $\M_1$ from $\D$
            \ENDIF
        \ENDWHILE
        \STATE Take $\hat \M \in \D$ and run $\hat \pi \in \argmax_{\pi \in \Pi} V_{\hat \M} (\pi)$ for the remaining episodes
    \end{algorithmic}
\end{small}
\end{algorithm}

\begin{algorithm}[t]
    \caption{Sampling Routine}
    \label{alg:mdp_sampling}
\begin{small}
    \begin{algorithmic}[1]
        \STATE \textbf{input} set of MDPs $\D$, state-action pair $(\bar s, \bar a)$, and visitation count $n$
        \STATE Initialize count $N_{\bar s \bar a} = 0$ and set $\X = \emptyset$
        \WHILE{$N_{\bar s \bar a} < n$}
            \FOR{$\M_i \in \D$}
                \STATE Run the policy $\pi_i \in \argmin_{\pi \in \Pi} \EV [X (\bar s | \M_i, \pi)]$ for two episodes
                \IF{$\bar s$ is reached}
                    \STATE Take action $\bar a$ and collect the next state $s'$
                    \STATE Update $N_{\bar s \bar a} = N_{\bar s \bar a} + 1, \X = \X \bigcup (s')$
                \ENDIF
            \ENDFOR
        \ENDWHILE
        \STATE \textbf{output} the set $\X = (s'_1, \ldots, s'_n)$ 
    \end{algorithmic}
\end{small}
\end{algorithm}

The procedure consists of two stages: An ``Identify'' stage aiming at identifying the test task with high probability (lines 2-11) and a ``Commit'' stage in which the collected information is exploited (line 12). The ``Identify'' stage works as follows. At each iteration, a pair of MDPs are drawn from the set of potential test tasks (line 3). A sampling routine (Algorithm~\ref{alg:mdp_sampling}) is invoked (line 5) to collect samples from the state-action pair where the transition models of the two tasks differ the most (see line 4). The collected information is used to eliminate the task that is less likely to be the test task within the drawn pair (lines 6-10). The ``Identify'' stage ends when the set of potential tasks $\D$ is reduced to a single candidate. The ``Commit'' stage then runs the optimal policy of the candidate task for the remaining episodes.

We can provide the following regret upper bound for Algorithm~\ref{alg:mdp_separation} by adapting~\citep[][Theorem 1]{chen2021understanding}.
\begin{restatable}[\citealt{chen2021understanding}]{theorem}{mdpSeparationUpperBound}
    Let $\M$ be a set of MDPs for which Assumption~\ref{ass:separation_mdp},~\ref{ass:reachability} hold. For any $\M_i \in \M$, we have
    \begin{equation*}
        R_H (\M_i, \mathbb{A}) = \cO \left( \frac{T M^2 \log (MH) \log^2 (SMH / \lambda)}{\lambda^4} \right)
    \end{equation*}
    where $\mathbb{A}$ is Algorithm~\ref{alg:mdp_separation} with inputs $\D = \M$ and $n = \frac{c \log^2 (S M H / \lambda) \log (MH)}{\lambda^4}$ for a sufficiently large constant $c$.
    \label{thr:mdp_upper_bound_separation}
\end{restatable}
The latter result is promising as it provides a fast rate with only logarithmic dependencies on $H$, but it also scales super-linearly with the size of the set $\M$, which is less than ideal for larger task sets. A natural question that arises is whether this is the best we can achieve under the considered separation condition. 
The latter is arguably a very important question as the separation condition is the minimal structural assumption that makes test-time regret minimization ``interesting'', statistically separating the problem from RL. 
In the next section, we provide an answer through a lower bound specifically designed for this setting.

\section{A Lower Bound for Test-Time Regret Minimization under Separation}
\label{sec:lower_bound}

In this section, we analyze the statistical barrier for test-time regret minimization under separation (Assumption~\ref{ass:separation_mdp}) by providing a novel lower bound.
Formally,
\begin{restatable}[Lower bound]{theorem}{mdpSeparationLowerBound}
    Let $\M$ be a set of MDPs for which Assumptions~\ref{ass:separation_mdp},~\ref{ass:reachability} hold. Let $T > M$ and $ M - 1 \leq H \leq C$ for some constant $C < \infty$.
    For any $\M_i \in \M$, algorithm $\mathbb{A}$, and confidence $\delta \in (0, 1)$, we have
    \begin{equation*}
        R_H (\M_i, \mathbb{A}) = \Omega \left( \frac{T M \log (H)}{\lambda} \log \left( \frac{1}{\delta} \right) \right)
    \end{equation*}
    with probability at least $1 - \delta$.
    \label{thr:mdp_lower_bound}
\end{restatable}
Interesting observations come through the lenses of the result above. The lower bound shows that the regret rate of Algorithm~\ref{alg:mdp_separation}~\citep{chen2021understanding} matches the optimal dependencies in $H, T$ factors, while it nearly matches the dependencies with $\lambda$ and the size of the set of tasks $M$. 
The factor of $1 / \lambda$ is not surprising, tying the complexity of the problem to how hard it is to distinguish one task from another.
The result has a fairly negative flavor on the dependency with $M$ instead. It demonstrates that the regret of any algorithm achieving fast rate $\log (H)$ has to scale at least linearly with $M$ under separation, essentially implying that those algorithms are unfit for large sets of tasks. 

Unfortunately, the latter settings in which the size of $\M$ may grow exponentially with the size of the tasks, or even be infinite (e.g., when tasks are continuous), are extremely relevant in practice, and it is arguably where the promises of meta RL for improved efficiency are the most enticing. 

\begin{figure*}[t]
    \centering
    \includegraphics[width=0.78\textwidth]{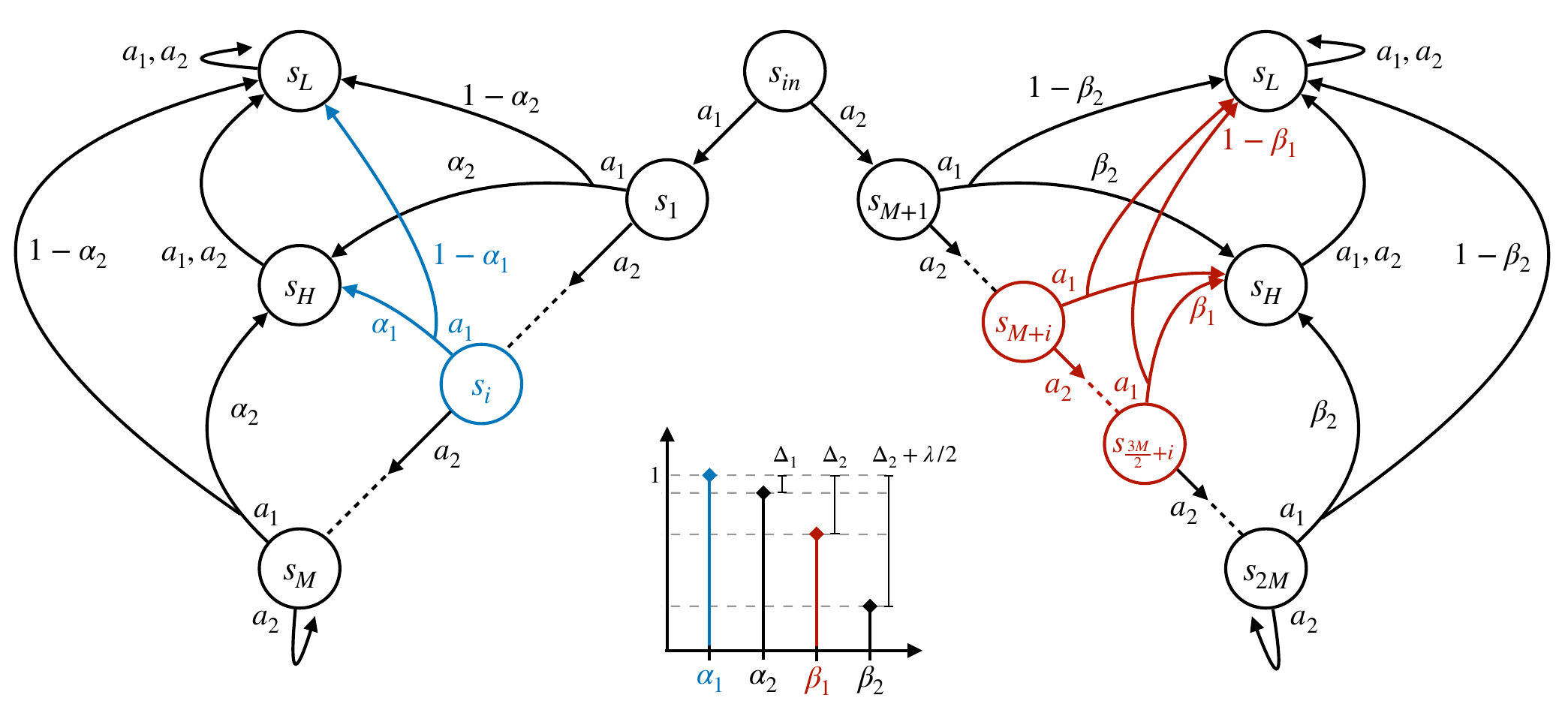}
    \vspace{-0.4cm}
    \caption{Visualization of the MDP $\M_i$ in the lower bound instance. Note that the role of state $s_i$ and $s_{M + i}, \ldots, s_{\frac{3M}{2} + i}$ change for every MDP in $\M$. Also note that $s_L, s_H$ on the left and right refer to the same pair of states, which are reported twice only to ease inspection. The bottom chart report the specification of the transition probabilities. The values of $\Delta_1, \Delta_2$ are designed to be small enough to make the optimal policy hard to identify playing only slightly sub-optimal policies and large enough to penalize easy identification, respectively.}
    \label{fig:lb_instance_mdp}
\end{figure*}

An open question remains on whether there exist relevant meta RL settings in which the structure of the problem can be further exploited to achieve fast rates $\log(H)$ while avoiding the dependence with $M$. In the next section, we discuss structural assumptions that go beyond the separation condition and allow to obtain the most efficient algorithms for test-time regret minimization.

Before that, we briefly sketch the main components of the proof of Theorem~\ref{thr:mdp_lower_bound}, which make use of original techniques and constructions that may be of independent interest. We defer thorough derivations to Appendix~\ref{sec:apx_proofs_lower_bound}. Note that a gentler introduction to the lower bound in the meta bandit setting is reported in Appendix~\ref{apx:bayesian_bandits}.

\subsection{Proof Sketch}

The key to our proof is to design a hard instance that links test-time regret minimization to the problem of best policy identification~\cite{fiechter1994efficient}, and then invokes a lower bound to the sample complexity of the latter to derive the result. While instance-dependent lower bounds of this kind exist in the literature~\citep[e.g.,][]{al2021navigating,wagenmaker2022beyond,al2023towards}, as a preliminary step we derive a result that is specifically tailored to our setting. Here we do not report derivations, which are non-trivial adaptations from a lower bound for BPI in infinite-horizon MDPs~\citep[see][Proposition 2]{al2021navigating}. We leave a detailed description to Appendix~\ref{apx:best_policy_identification}.

In Figure~\ref{fig:lb_instance_mdp}, we report a visualization of the instance constructed to derive the lower bound, which consists of $M$ MDPs having $2M + 3$ states and $2$ actions each, a high-rewarding state $s_H$, and an absorbing state $s_L$ with zero reward. It is easy to see that the optimal policy for $\M_i$ goes to the state $s_i$ and then takes action $a_1$, which gives the highest probability to visit $s_H$. However, taking action $a_1$ in every other state $s_j \in (s_1, \ldots, s_M) \setminus s_i$ is only slightly sub-optimal, meaning that regret minimization is hard. Instead, it is easier to identify the test task $\M_i$ first, by playing action $a_1$ in the states $s_j \in (s_{M + 1}, \ldots, s_{2M})$, for which at least one $(s_j, a_1)$ pair is guaranteed to reveal the test task against any other, and then to play the optimal policy thereafter.  
To formalize the latter intuition, we note that identifying the test task is equivalent to a BPI problem for this instance and we center the proof around the event
\begin{equation*}
    \mathcal{E} := \big\{\text{``best policy is identified within $H$ episodes''}\big\}.
\end{equation*}
Then, we show that the regret is lower bounded by $\Omega (\sqrt{H})$ when $\mathcal{E}$ does not hold, which implies that solving the BPI problem is necessary to obtain the best rate.\footnote{Note that this does not prescribe how the algorithm collect samples. It tells that, whatever the sampling strategy, the BPI problem has to be solved within $H$ episodes.} At this point, we invoke the BPI lower bound in Lemma~\ref{lemma:best_policy}, which simultaneously guarantees that $\mathcal{E}$ holds with probability at least $1 - \delta$ and that the regret is lower bounded by
\begin{equation*}
    R_H (\M_i, \mathbb{A}) \geq \EV [\tau] \Delta_2
\end{equation*}
where $\EV [\tau] \geq TM / \lambda^2 \log (1 / 2.4 \delta)$ is the sample complexity of the BPI problem on the constructed instance and $\Delta_2 := V^*_i - V (\pi) = \log (H) / \sqrt{H}$ is the value gap of playing a sub-optimal policy to identify the task. Theorem~\ref{thr:mdp_lower_bound} is obtained through additional algebraic manipulations.

\section{Strong Identifiability: Beyond Separation for Faster Rates}
\label{sec:beyond_separation}

In the previous section, we have settled the statistical complexity of test-time regret minimization in meta RL with a finite set of tasks under a common separation condition (Assumption~\ref{ass:separation_mdp} has been previously considered in~\citealt{chen2021understanding} as is, while a stronger version of the condition has been used in~\citealt{kwon2021rl}). 

Especially, we provided a lower bound to the test-time regret of order $\Omega (T M \log (H) / \lambda)$ and we showed that an ``Identify-then-Commit'' strategy~\citep{chen2021understanding} leads to a nearly matching upper bound $\cO (T M^2 \log(M H) / \lambda^4)$. Most importantly, our analysis shows that a linear dependence with the size $M$ of the set of tasks cannot be avoided under the considered separation condition, at least if we aim for a fast rate in $H$. In this section, we present stronger yet reasonable assumptions that allow for faster rates of the test-time regret. 
These assumptions implicitly discriminate the problem instances, which we call \emph{strongly identifiable}, that are truly worth framing into a meta RL paradigm, as the latter provides undeniable statistical benefits over RL.
\begin{definition}[Strong identifiability]
    Let $\M$ be a set of MDPs for which Assumption~\ref{alg:mdp_separation} hold. $\M$ is \emph{strongly identifiable} if, for all $\M_i \in \M$, we can identify $\M_i$ with high probability by only playing policies from a restricted policy class $\Pi^{id}$ where $|\Pi^{id}| \leq \gamma M$ for some factor $\gamma \in (0, 1)$.
\end{definition}
Next, we provide three instances of strong identifiability, with specialized algorithms and corresponding analyses, while investigating a broader range of strongly identifiable problems is an interesting directions for future works.


\subsection{Meta RL with Clustering}
\label{sec:beyond_separation_clustering}

Let us assume that the tasks in $\M$ can be grouped in coherent \emph{clusters} admitting a peculiar and testable property that allows to efficiently discriminate one cluster from the others. Formally,
\begin{assumption}[$KN\lambda$-clustering]
\label{ass:mdp_clustering}
    Let $\C = (\C_k)_{k \in [K]}$ be a partition of $\M$ such that $|\C_k| \leq N$ for all $k \in [K]$ and $N > K$. For all $\C_k \in \C$, there exists $(s,a) \in \S \times \A$ such that $\min_{\M_i \in \C_k} \min_{\M_j \in \M \setminus \C_k} \| (p_i - p_j)(\cdot | s, a) \|_1 \geq \lambda.$
\end{assumption}
Further, we consider a slightly stronger reachability condition within a cluster, which essentially assures that a state can be reached throughout the cluster with the same policy.
\begin{assumption}[Cluster reachability]
\label{ass:mdp_clustering_reachability}
    Let $C_k \in \C$ be a cluster of reachable MDPs. For all $\M_i, \M_j \in \C_k$ and $s \in \S$, it holds $\EV [X (s | \M_j, \pi_i)] \leq T / 2$ for $\pi_i \in \argmin_{\pi \in \Pi} \EV [X (s | \M_i, \pi)]$.
\end{assumption}

The combination of latter assumptions allows to take a different algorithmic approach w.r.t.~the standard separation condition (Assumption~\ref{ass:separation_mdp}). Essentially, we can split the ``Identify'' stage of Algorithm~\ref{alg:mdp_separation} into two separate phases: First, we identify the cluster to which the test task belongs with high probability, then, we call Algorithm~\ref{alg:mdp_separation} to identify the test task within the cluster selected in the previous step. We call this procedure ``Double-Identify-then-Commit'' to underline the two-phase structure of the ``Identify'' stage. We report the pseudocode of the procedure in Algorithm~\ref{alg:mdp_clustering}.

\begin{algorithm}[t]
    \caption{Double-Identify-then-Commit}
    \label{alg:mdp_clustering}
\begin{small}
    \begin{algorithmic}[1]
        \STATE \textbf{input} set of MDPs $\M$, clusters $\C$, visitation count $n_{\C}$
        \STATE Initialize $\D = (\M_k)_{k \in [K]}$ with $\M_k$ from $\C_k$ randomly
        \WHILE{$|\C| > 1$}
            \STATE Draw $\C_k$ from $\C$ randomly and let $(\M_1, \M_2, \bar s, \bar a) \in \argmin\limits_{\M_i \in \C_k, \M_j \in \M \setminus \C_k} \max\limits_{(s, a) \in \S \times \A} \| (p_i - p_j) (\cdot | s, a) \|_1$
            \STATE Call  Algorithm~\ref{alg:mdp_sampling} with $\D, (\bar s, \bar a), n_\C$ to collect $\X$
            \IF{$\exists s' \in \X: p_1 (s' | \bar s, \bar a) = 0$ or $\prod_{s' \in \X} \frac{p_1 (s' | \bar s, \bar a)}{p_2 (s' | \bar s, \bar a)} \geq 1$}
                \STATE Update $\C = \C_k$
            \ELSE
                \STATE Update $\C = \C \setminus \C_k$
            \ENDIF
        \ENDWHILE
        \STATE Take $\hat \C \in \C$ and run Algorithm~\ref{alg:mdp_separation} with $\D = \hat \C, n = \frac{c \log^2 (S N H / \lambda) \log (NH)}{\lambda^4}$ for the remaining episodes
    \end{algorithmic}
\end{small}
\end{algorithm}

Since both the number of clusters $K$ and the size of each cluster are smaller than the size $M$ of the set of MDPs $\M$ under Assumption~\ref{ass:mdp_clustering}, the set of MDPs is strongly identifiable. Thus, we can expect the ``Identify'' stage of Algorithm~\ref{alg:mdp_clustering} to be more efficient than the same stage of Algorithm~\ref{alg:mdp_separation} taking $\M$ as input, which is confirmed by the result below.
\begin{restatable}[]{theorem}{mdpClusteringUpperBound}
    Let $\M$ be a set of MDPs for which Assumption~\ref{ass:reachability},~\ref{ass:mdp_clustering},~\ref{ass:mdp_clustering_reachability} hold. For any $\M_i \in \M$, we have
    \begin{equation*}
        R_H (\M_i, \mathbb{A}) = \cO \left( \frac{T (K^2 + N^2) \log (N H) \log^2 (\frac{S N H}{\lambda})}{\lambda^4} \right)
    \end{equation*}
    where $\mathbb{A}$ is Algorithm~\ref{alg:mdp_clustering} with inputs $\M, \C$, and $n_\C = \frac{c \log^2 (S K H / \lambda) \log (KH)}{\lambda^4}$ for a sufficiently large constant $c$.
    \label{thr:mdp_upper_bound_clustering}
\end{restatable}

While the latter proves a significant speed-up of the ``Identify'' stage, which can be relevant in practical settings, the regret does not essentially escape the quadratic dependency on the number of tasks $M$. Indeed, if we look into the rate we get $K^2 + N^2 \approx K^2 + M^2 / K^2$ as opposed to $M^2$ with the separation condition (Assumption~\ref{ass:separation_mdp}) alone.
At this point, it is interesting to ask which kind of structure can lead to a regret rate that is truly sublinear in the number of tasks $M$. In the next section, we achieve that by considering a tree structure on the set of MDPs $\M$.

\subsection{Meta RL with a Tree Structure}
\label{sec:beyond_separation_tree_structure}

Let us suppose that for any set of candidate test MDPs we extract from the original set there always exist a state-action pair from which we can collect information to \emph{split} the candidates in two subsets. For instance, we can think of a set of indoor physical domains with peculiar configurations, e.g., some with a window on the right-hand side and some without. If those states not only exist, but they are also easy to reach, we can then sequentially split the original set of MDPs to increasingly smaller subsets, until a single candidate remains. To formalize this intuition, we first define what we mean for a state-action to be \emph{easy-to-reach}.
\begin{assumption}[Strong reachability]
\label{ass:strong_reachability}
	Let $\D \subseteq \M$ a set of MDPs. We say that $(s, a) \in \S \times \A$ is \emph{strongly reachable} in $\D$ if, for all $\M_i, \M_j \in \D$, it holds  $\EV [ X(s, a | \M_j, \pi_i)] \leq T / 2$ where $\pi_i \in \argmin_{\pi \in \Pi} \EV[X(s, a | \M_i, \pi)]$.
\end{assumption}

\begin{algorithm}[t]
    \caption{Tree-Identify-then-Commit}
    \label{alg:mdp_tree}
\begin{small}
    \begin{algorithmic}[1]
        \STATE \textbf{input} set of MDPs $\M$ and visitation count $n$
        \STATE Initialize $\D = \M$
        \WHILE{$|\D| > 1$}
            \STATE Compute $\M_1, \M_2, \D^+, \D^-, (\bar s, \bar a)$ by solving
            \vspace{-5pt}
            \begin{align*}
                &\min_{\M_i \in \D^+, \M_j \in \D^-} \max_{(s, a) \in \S \times \A} \| (p_i - p_j) (\cdot | s, a) \|_1 \\
                &\text{ s. t. } \D^+ \cup \D^- = \D, (s, a) \text{ is strongly reachable}
            \end{align*}
            \STATE Initialize $N_{\bar s \bar a} = 0$ and $\X = \emptyset$
            \WHILE{$N_{\bar s \bar a} < n$}
                \STATE Run the policy $\pi \in \argmin_{\pi \in \Pi} \EV [X (\bar s | \M_1, \pi)]$ for two episodes
                \IF{$\bar s$ is reached}
                    \STATE Take action $\bar a$ and collect the next state $s'$
                    \STATE Update $N_{\bar s \bar a} = N_{\bar s \bar a} + 1, \X = \X \bigcup (s')$
                \ENDIF
            \ENDWHILE
            \IF{$\exists s' \in \X: p_1 (s' | \bar s, \bar a) = 0$ or $\prod_{s' \in \X} \frac{p_1 (s' | \bar s, \bar a)}{p_2 (s' | \bar s, \bar a)} \geq 1$}
                \STATE Update $\D = \D^+$
            \ELSE
                \STATE Update $\D = \D^-$
            \ENDIF
        \ENDWHILE
        \STATE Take $\hat \M \in \D$ and run $\hat \pi \in \argmax_{\pi \in \Pi} V_{\hat \M} (\pi)$ for the remaining episodes
    \end{algorithmic}
\end{small}
\end{algorithm}

Then, we can formally define our structural assumption.
\begin{assumption}[$\beta$-tree]
\label{ass:mdp_tree}
    Let $\M$ be a set of MDPs and let $\beta \in (\nicefrac{1}{2},1)$. For all $\D \subseteq \M$, there exist a strongly reachable $(s,a) \in \S \times \A$ and a partition $(\D^+, \D^-)$ of $\D$ such that $\min_{\M_i \in \D^+} \min_{\M_j \in \D^-} \| (p_i - p_j) (\cdot |s,a) \|_1 \geq \lambda$ where $\max (|\D^+|, |\D^-|) / |\D| \leq \beta$.
\end{assumption}

The latter induces a tree structure on the set of MDPs, such that we can render the identification problem as traversing a (binary) decision tree. The routine works as follows, we start with the initial set of MDPs $\M$, then we look for a state-action pair inducing a large enough split, which we visit several times to understand whether the test task belongs to $\D^+$ or $\D^-$. We can iterate these steps again on the resulting set, i.e., $\D^+$ or $\D^-$, so that the set of candidate MDPs iteratively shrinks towards one single MDP, which is our candidate test task for the remaining episodes. The pseudocode for the procedure is reported in Algorithm~\ref{alg:mdp_tree}.

Clearly, the cost of the ``Identify'' stage of Algorithm~\ref{alg:mdp_tree}, i.e., the number of times the while loop between lines 2-12 has to be executed, is tied to the depth of the tree structure, which is $\log_{1 / \beta} (M)$ in the worst case. With this consideration, we can derive the following regret upper bound.
\begin{restatable}[]{theorem}{mdpTreeUpperBound}
    Let $\M$ be a set of MDPs for which Assumption~\ref{ass:strong_reachability},~\ref{ass:mdp_tree} hold and let $d = \log_{1 / \beta} (M)$. For any $\M_i \in \M$, we have
    \begin{equation*}
        R_H (\M_i, \mathbb{A}) = \cO \left( \frac{T d \log (d H) \log^2 (S d H / \lambda) }{\lambda^4} \right)
    \end{equation*}
    where $\mathbb{A}$ is Algorithm~\ref{alg:mdp_tree} with inputs $\M$ and $n = \frac{c \log^2 (S d H / \lambda) \log (d H)}{\lambda^4}$ for a sufficiently large constant $c$.
    \label{thr:mdp_upper_bound_tree}
\end{restatable}

\subsection{Meta RL with a few Revealing Policies}
\label{sec:beyond_separation_revealing}


What happens if we assume that we can extract from $\M$ a small set of \emph{revealing} policies that allow to traverse all of the revealing state-action pairs in expectation? We formalize this intuition in the following assumption. Then, we provide an algorithm and corresponding regret analysis that shows the quadratic dependence with $M$ can be escaped.

Before stating the assumption, we define the set of revealing state-action pairs $\overline{\S\A} \subseteq{\S \times \A}$ such that 
    $\forall \M_i, \M_j \in \M, \exists (s,a) \in \overline{\S\A} : \| (p_i - p_j) (\cdot | s, a) \|_1 \geq \lambda$. Then, 
\begin{assumption}[Revealing policy set]
\label{ass:mdp_revealing_policies}
    Let $\M$ be a set of MDPs for which Assumption~\ref{ass:separation_mdp},~\ref{ass:reachability} hold. There exists a set of policies $\Pi_I$ of size $|\Pi_I| \leq I$ such that $\forall \M_i \in \M$ it holds
    $\max_{\pi \in \Pi_I} \min_{(s, a) \in \overline{\S\A}} P (X (s, a | \M_i, \pi) \leq T) \geq 1 / 2.$
\end{assumption}
In Algorithm~\ref{alg:mdp_revealing_policies}, we report the pseudocode of a procedure that exploits Assumption~\ref{ass:mdp_revealing_policies} to increase the efficiency of the data collection for the ``Identify'' stage. First, the policies in $\Pi_I$ are repeatedly deployed with the goal of collecting $n$ samples for each revealing state-action pair in $\overline{\S\A}$ within an ``Explore'' stage (lines 3-11). Then, differently from previous approaches, the ``Identify'' stage (lines 13-21) is carried out offline with the previously collected data, until the set of tasks $\D$ is reduced to one candidate. Finally, in the ``Commit'' stage (line 22), the optimal policy $\hat \pi$ in the identified test task $\hat \M$ is deployed for the remaining episodes.

\begin{algorithm}[t]
    \caption{Explore-Identify-then-Commit}
    \label{alg:mdp_revealing_policies}
\begin{small}
    \begin{algorithmic}[1]
        \STATE \textbf{input} set of MDPs $\M$, set of policies $\Pi_I$, set of state-action pairs $\overline{\S\A}$, visitation count $n$
        \STATE Initialize $N_{\bar s \bar a} = 0$ and $\X_{\bar s \bar a} = \emptyset$ for all $(\bar s, \bar a) \in \overline{\S \A}$
        \WHILE{$\min ( N_{\bar s \bar a}) < n$}
            \FOR{$\pi_i \in \Pi_I$}
                \STATE Run the policy $\pi_i$ for two episodes
                \IF{$\bar s$ such that $(\bar s, \cdot) \in \overline{\S\A}$ is reached}
                    \STATE Take action $\bar a$ and collect the next state $s'$
                    \STATE Update $N_{\bar s \bar a} = N_{\bar s \bar a} + 1$ and $\X_{\bar s \bar a} = \X_{\bar s \bar a} \bigcup (s')$
                \ENDIF
            \ENDFOR
        \ENDWHILE
        \STATE Initialize $\D = \M$
        \WHILE{$|\D| > 1$}
            \STATE Draw $\M_1, \M_2$ from $\D$ randomly
            \STATE Let $(\bar s, \bar a) \in \argmax_{(s, a) \in \overline{\S\A}} \| (p_1 - p_2) (\cdot | s, a) \|_1$
            \IF{$\exists s' \in \X_{\bar s \bar a} : p_2 (s' | \bar s, \bar a) = 0$ or $\hspace{-0.1cm} 
            \raisebox{0.1cm}{$\prod\limits_{s' \in \X_{\bar s \bar a}}$}
            \frac{p_1 (s' | \bar s, \bar a)}{p_2 (s' | \bar s, \bar a)} \geq 1$}
                \STATE Eliminate $\M_2$ from $\D$
            \ELSE
                \STATE Eliminate $\M_1$ from $\D$
            \ENDIF
        \ENDWHILE
        \STATE Take $\hat \M \in \D$ and run $\hat \pi \in \argmax_{\pi \in \Pi} V_{\hat \M} (\pi)$ for the remaining episodes
    \end{algorithmic}
\end{small}
\end{algorithm}

Now, we provide an upper bound to its test-time regret.

\begin{restatable}[]{theorem}{mdpRevealingUpperBound}
    \label{thr:mdp_upper_bound_revealing_policies}
    Let $\M$ be a set of MDPs for which Assumption~\ref{ass:mdp_revealing_policies} holds. For any $\M_i \in \M$, we have
    \begin{equation*}
        R_H (\M_i, \mathbb{A}) = \cO \left( \frac{T I \log (M H) \log^2 (S M H / \lambda)}{\lambda^4} \right)
    \end{equation*}
    where $\mathbb{A}$ is Algorithm~\ref{alg:mdp_revealing_policies} with inputs $\M, \Pi_I, \overline{\S\A}$, $n = \frac{c \log^2 (S M H / \lambda) \log (MH)}{\lambda^4}$ for a sufficiently large constant $c$.
\end{restatable}

The latter result provides a fast rate for the test-time regret when $I < M^2$, i.e., the problem is strongly identifiable.
However, Algorithm~\ref{alg:mdp_revealing_policies} takes a set of revealing policies as input, leaving to the training phase the burden of providing it. Here we discuss briefly how this requirement can be avoided, while we leave as future work a thorough investigation.

\textbf{Sampling from revealing policies.}~~
Instead of pre-computing a set of revealing policies, one may try to replicate the sampling of the ``Explore'' stage (lines 3-11) while interacting with the test task. In Algorithm~\ref{alg:mdp_revealing_policies_sampling} we describe an adaptive sampling procedure for this purpose. The idea is to iteratively compute a policy that maximizes the (expected) visitation of the revealing state-action pairs left uncovered. The latter objective can be encoded into a trajectory reward defined as in line 4. Once such a policy is computed (line 5), it can be deployed in the test task to collect a trajectory (line 6) and to update the set of the remaining state-action pairs to be covered (line 7). The process repeats until of the revealing state-action pairs have been visited (lines 3-9).

To implement Algorithm~\ref{alg:mdp_revealing_policies_sampling}, two technical hurdles have to be overcome. First, the trajectory reward in line 4 is non-standard. A naive solution is to encode in the state space the necessary information to define the reward on a state level or, alternatively, to adapt algorithms to deal with trajectory rewards~\cite{efroni2021reinforcement, chatterji2021theory}. Secondly, the problem in line 5 is akin to a robust MDP, which is known to be intractable in general~\cite{wiesemann2013robust}. To address the latter, a policy gradient method~\cite{kumar2024policy} can be used to get an approximate solution efficiently, although this may slow down convergence of the while loop (lines 3-9) as opposed to deploying (in line 6) the exact solution of the problem.

\begin{algorithm}[t]
\begin{small}
    \caption{Revealing Policies Sampling}
    \label{alg:mdp_revealing_policies_sampling}
    \begin{algorithmic}[1]
        \STATE \textbf{input} set of MDPs $\M$, set of state-action pairs $\overline{\S\A}$
        \STATE Initialize $h \gets 1$ and the sets $\overline{\S\A}_1 = \overline{\S\A}$, $\X_{\bar s \bar a} = \emptyset$
        \WHILE{$\overline{\S\A}_{h}$ is not empty}
            \STATE Define $ \widetilde r (\tau) := \hspace{-0.3cm} \sum\limits_{(s,a) \in \overline{\S\A}_h} \hspace{-0.3cm}  \mathds{1} \big( \exists (s_t, a_t) \in \tau : (s_t, a_t) = (s, a) \big) $
            \STATE Solve
            $ \pi_h \in \argmax_{\pi \in \Pi} \min_{\M_i \in \M} \EV_{\tau \sim \pi, \M_i} [ \widetilde r (\tau) ]$
            \STATE Take trajectory $\tau_h$  with $\pi_h$ and populate $\X_{\bar s \bar a}$
            \STATE Compute
            $\overline{\S\A}_{h + 1} = \overline{\S\A}_h \setminus \big( (s,a) \in \overline{\S\A}_h : (s,a) \in \tau_h \big)$
            \STATE Increment $h \gets h + 1$
        \ENDWHILE
        \STATE \textbf{output} the sets $\X_{\bar s \bar a}$
    \end{algorithmic}
\end{small}
\end{algorithm}

\section{Related Works}
\label{sec:related_works}

While we are not aware of any previous work explicitly addressing test-time regret minimization under perfect training within the meta RL paradigm, slight variations of our problem setting have been considered in different domains.

As we extensively reported in the previous sections, \citet{chen2021understanding} provides theoretical results on the sim-to-real gap in domain randomization that can be transferred to our setting almost verbatim. Their \emph{sim} stage coarsely correspond to our training, while their \emph{real} stage is our test task, for which they study a notion of regret against the optimal policy specific to the task. Differently from our setting, they consider infinite-horizon MDPs and they analyse the regret rate of a Bayes-optimal policy instead of our adaptive algorithms. Notably, they assume to have recovered an exact Bayes-optimal policy from simulations, which is similar in nature to our perfect training assumption. In their setting of interest, they provide regret guarantees of order $\cO(M^2 \log (MH))$ for finite set of tasks under separation, $\cO(M^2 \sqrt{H})$ for finite set of tasks without separation, and $\cO(\sqrt{d_E H})$ for infinite set of tasks with function approximation, where $d_E$ is the eluder dimension of the function class. Finally, they provide a lower bound $\Omega (H)$ for finite sets without separation. We fill the gaps in their analysis providing a lower bound specialized for the separation condition and assumptions beyond separation for faster rates.

The work by~\citet{ye2023power} studies generalization guarantees of pre-training in RL, in terms of Bayesian and frequentist regret, zero-shot generalization or with additional test-time interaction. The latter setting in the frequentist regret formulation is the analogous to our test-time regret minimization, for which they provide a policy-collection elimination algorithm with regret of order $\cO(\sqrt{\mathbb{C} (P) H})$, where $\mathbb{C} (P)$ is a measure of complexity of the task distribution. Although the complexity $\mathbb{C} (P)$ can capture structured finite or infinite set of tasks, it does not allow for escaping the $\sqrt{H}$ regret of standard RL.

\citet{kwon2021rl} address regret minimization in Latent MDPs (LMDPs). In their setting, at every episode a task is drawn from a set of finite but unknown tasks, for which the agent tries to minimize the regret against an optimal policy for the specific task. Essentially, LMDPs formalism can be seen as a variation of our setting in which the test task is not persistent but changes at every episode, and the agent does not have full knowledge of the transition dynamics of the tasks, which have to be estimated from samples. In its full generality, LMDPs are statistically intractable. Analogously to our work, they consider a (stronger) version of the separation condition to achieve a regret rate of order $\cO(\sqrt{MH})$ for their inherently harder setting. Further variations of LMDPs have been studied, including reward-mixing MDPs~\cite{kwon2021reinforcement, kwon2023reward}, analogous of LMDPs with fixed dynamics but changing rewards, LMDPs with side information partially revealing the current task~\cite{kwon2023prospective}, and mixture of MDPs~\cite{kausik2023learning}.

\section{Conclusion}

In this paper, we provided a formal study on the statistical barriers of test-time regret minimization under strong structural assumptions, shedding light on when meta RL can be expected to provide significant benefits over standard RL.

First, we settled the complexity of test-time regret minimization under separation deriving a lower bound specialized for the assumption, for which only upper bounds were known in the literature. Then, to overcome the (super)linear dependency with the size of the set of tasks, we studied a family of structural assumptions beyond separation, i.e., the set of tasks can be grouped in coherent clusters, the test task can be identified following a tree structure, or a small set of revealing policies can be deployed to identify the test task.

Future works may extend our results in various directions. Additional structural assumptions fitting in the broad family of strong identifiability may be investigated, as well as separation conditions at the level of trajectory generation processes rather than single state-action pairs. Understanding the impact of approximate training, i.e., only imperfect estimates of the tasks' transition dynamics are available, on our test-time regret minimization results is important to bring our analysis in a more realistic setting. Whether there exists minimal assumptions that allow for fast rates of order $\log (H)$ for infinite set of tasks is also a question worth investigating. 

Finally, we hope that our theoretical study can bring inspiration to design practical algorithms for improved efficiency of test-time learning in meta RL.

\section*{Acknowledgments}

This research was Funded by the European Union (ERC, Bayes-RL, 101041250). Views and opinions expressed are however those of the author(s) only and do not necessarily reflect those of the European Union or the European Research Council Executive Agency (ERCEA). Neither the European Union nor the granting authority can be held responsible for them.

The authors thank Shie Mannor for insightful discussions on an early draft of this work.


\section*{Impact Statement}
This paper presents work whose goal is to advance the field of Machine Learning. There are many potential societal consequences of our work, none which we feel must be specifically highlighted here.

\bibliography{biblio}
\bibliographystyle{icml2024}

\newpage
\appendix
\onecolumn

\section{Missing Proofs}
\label{apx:proofs}

In this section, we provide complete derivations to prove the theoretical results presented in the paper.

\subsection{Proof of Theorem~\ref{thr:mdp_upper_bound_separation}}
\label{sec:apx_proofs_separation}

Here we prove the upper bound to the test-time regret under separation of Algorithm~\ref{alg:mdp_separation}, which is a straightforward adaptation of the derivations in~\citet[][Theorem 5]{chen2021understanding} to the finite-horizon setting.

\mdpSeparationUpperBound*
\begin{proof}
    Analogously as in~\citet[][Theorem 5]{chen2021understanding}, the proof is based on showing that the true MDP $\M_i$ will not be eliminated from $\D$ (lines 6-10 in Algorithm~\ref{alg:mdp_separation}) with probability at least $1 - 1 / H$~\citep[][Lemma 4]{chen2021understanding}. Especially, we can write
    \begin{align*}
        P \Big( &\text{``$\M_i$ is eliminated from $\D$''} \Big) \\
        &= P \left( \bigcup_{m = 1}^{M - 1} \text{``$\M_i$ is eliminated from $\D$ at iteration $m$''} \right) \\
        &\leq \sum_{m = 1}^{M - 1} P \Big( \text{``$\M_i$ is eliminated from $\D$ at iteration $m$''} \Big) 
    \end{align*}
    by noting that the loop in lines 2-11 of Algorithm~\ref{alg:mdp_separation} is executed for $M - 1$ iterations and then applying a union bound. Next, we call Lemma~\ref{lemma:mdp_elimination} to prove that the event ``$\M_i$ is eliminated from $\D$ at iteration $m$'' holds with probability less than $1 / MH$.\footnote{Note that the derivations in the corresponding~\citep[][Lemma 5]{chen2021understanding} apply verbatim as there is no assumption on how samples in $\X$ are collected, and whether they are coming from a finite-horizon or an infinite-horizon MDP.}
    
    Then, we just need to prove that those $n$ samples rquired by Lemma~\ref{lemma:mdp_elimination} can be collected efficiently through the sampling routine in Algorithm~\ref{alg:mdp_sampling}, which is where our approach differs from~\citet{chen2021understanding}. In Lemma~\ref{lemma:sampling_routine}, we provide a quick adaptation of their infinite-horizon communicating MDP setting to our finite-horizon reachable MDP setting. 
    
    Now let $H_0$ the number of episodes needed to collect $n$ samples through Algorithm~\ref{alg:mdp_sampling} for every time is called from Algorithm~\ref{alg:mdp_separation} (line 5), which is $M - 1$ times in total. We have
    \begin{equation*}
        \EV [H_0] = h_0 \leq 2 (M - 1) M n \leq \frac{c M^2 \log^2 (SMH / \lambda) \log (MH)}{\lambda^4}.
    \end{equation*}
    Finally, we can write
    \begin{align*}
        R_H (\M_i, \pi) &= \EV \left[ \sum_{h = 1}^{h_0} V_i^* - V_i (\pi_h) \right] + \EV \left[ \sum_{h = h_0}^H V_i^* - V_i (\hat \pi) \right] \leq \frac{c T M^2 \log^2 \left(\frac{SMH}{\lambda}\right) \log (MH)}{\lambda^4}
    \end{align*}
    by noting that $\EV [ \sum_{h = 1}^{h_0} V_i^* - V_i (\pi_h) ] \leq \EV[H_0] T$ through $r_i (s, a) \in [0, 1], \forall \M_i \in \M$, and that $\EV [ \sum_{h = h_0}^H V_i^* - V_i (\hat \pi) ] \leq T$, as it is $V_i^* - V_i (\hat \pi) = 0$ with probability at least $1 - 1 / H$ and $V_i^* - V_i (\hat \pi) \leq T H$ with probability at most $1 / H$.
\end{proof}

\begin{lemma}[\citealt{chen2021understanding}]
\label{lemma:mdp_elimination}
    Let $\X = (s'_1, \ldots, s'_n)$ be a set of $n = \frac{c \log^2 (S M H / \lambda) \log (MH)}{\lambda^4}$ independent samples from $p_i (\cdot | \bar s, \bar a)$ for a large enough constant $c$ and let $\M_1$ be an MDP such that $\| (p_i - p_1) (\cdot|\bar s, \bar a) \|_1 \geq \lambda$. Then, it holds
    \begin{equation*}
        \prod_{s' \in \X} \frac{p_i (s'| \bar s, \bar a)}{p_1 (s'| \bar s, \bar a)} > 1
    \end{equation*}
    with probability at least $1 - 1 / MH$.
\end{lemma}

\begin{lemma}
\label{lemma:sampling_routine}
    Let $\M_i \in \D$ an MDP and let $H_0$ a random variable denoting the number of episodes needed by Algorithm~\ref{alg:mdp_sampling} to collect $n$ samples from $p_i (\cdot | \bar s, \bar a)$ in $\M_i$. We can upper bound the expected number of episodes as $h_0 := \EV [H_0] \leq 2 M n$.
\end{lemma}
\begin{proof}
     We can follow similar steps as in~\citet[][Lemma 7]{chen2021understanding}. From Assumption~\ref{ass:reachability} we have that $\EV (X (\bar s|\M_i, \pi_i)) \leq T / 2$ for $\pi_i \in \argmin_{\pi \in \Pi} \EV [X (\bar s | \M_i, \pi)]$. Then, by applying the Markov's inequality $P(X(\bar s | \M_i, \pi_i) \geq T) \leq \EV[X(\bar s | \M_i, \pi_i)] / T$ we get 
    \begin{equation}
        X (\bar s | \M_i, \pi_i) \leq T
        \label{eq:Markov_inequality}
    \end{equation}
    with probability at least $1 / 2$. Let $Y$ be the random variable denoting the number of episodes needed to reach state $\bar s$. From~\eqref{eq:Markov_inequality}, we have that $P(Y = k) \leq 1 / 2^k$, which gives $\EV[Y] \leq \sum_{k = 1}^{\infty} k / 2^k = 2$. Since Algorithm~\ref{alg:mdp_sampling} deploys $\pi_j \in \argmin_{\pi \in \Pi} \EV [X (\bar s | \M_j, \pi)]$ for all $\M_j \in \D$ (lines 4-9), then also $\pi_i$ is deployed.
\end{proof}

\subsection{Proof of Theorem~\ref{thr:mdp_lower_bound}}
\label{sec:apx_proofs_lower_bound}

Here we provide complete derivations for the proof of the lower bound to the test-time regret under separation, which we briefly sketched in Section~\ref{sec:lower_bound}. In the proof, we will refer to the hard instance depicted in Figure~\ref{fig:lb_instance_mdp} of the main paper.

\mdpSeparationLowerBound*
\begin{proof}
    The key idea behind this proof is to construct an instance of the problem in which it is hard to minimize the regret without knowing the MDP, whereas it is easy to identify the MDP playing sub-optimal policies. 

    \paragraph{Hard instance.} The instance consists of $M$ MDPs having $2M + 3$ states and $2$ actions each. Figure~\ref{fig:lb_instance_mdp} depicts the sample MDP $\M_i \in \M$, but all of the MDPs in the instance are similarly constructed. For any $\M_i \in \M$, the state $s_{in}$ is the initial state such that $p (s_1 | s_{in}, a_1) = p (s_{M + 1} | s_{in}, a_2) = 1$, $s_L$ is an absorbing state with $p(s_L | s_L, a_1) = p(s_L | s_L, a_2) = 1$, $s_H$ is an high-reward state such that $p (s_L | s_H, a_1) = p (s_L | s_H, a_2) = 1$ and $r (s_H, a_1) = r(s_H, a_2) = 1$. For all the other states $s \in \S \setminus (s_H)$ we have $r (s, a_1) = r (s, a_2) = 0$. The states $s \in \S \setminus (s_H, s_L, s_{in})$ are arranged in two different chains: $(s_1, \ldots, s_i, \ldots, s_M)$ on the left and $(s_{M + 1}, \ldots, s_{M + i}, \ldots, s_{2M})$ on the right, respectively. In every state of those chains, the action $a_2$ gives a deterministic transition to the next state in their respective chain, i.e., $p(s_{j + 1} | s_j, a_2) = 1, \forall s_j \in (s_1, \ldots, s_{M - 1}) \cup (s_{M + 1}, \ldots, s_{2M - 1})$ and self-loops $p (s_M | s_M, a_2) = p(s_{2M} | s_{2M}, a_2) = 1$ for the closing end of the chains. For any $\M_i \in \M$, the state $s_i$ is the one leading to the state $s_H$ with higher probability than all of the other states $p_i (s_H | s_i, a_1) = 1$. For all of the other states in the left chain, i.e., $s \in (s_1, \ldots, s_M) \setminus (s_i)$, the transition to $s_H$ has probability $p_i (s_H | s, a_1) = 1 - \Delta_1 = 1 - 1 / \sqrt{H}$. In the right chain, the states are divided in two groups of $M/2$ states, which are $\mathcal{G}_1 := (s_{M + i}, \ldots, s_{\frac{3M}{2} + i})$ and $\mathcal{G}_2 := (s_{M + 1}, \ldots, s_{M + i - 1}) \cup (s_{\frac{3M}{2} + i}, \ldots, s_{2M})$.\footnote{Here we consider $i \leq M / 2$ for the sake of clarity. If $i > M /2$ some of the indices of $\mathcal{G}_1$ will exceed $2M$, so that the exceeding states are to be taken from the start of the right chain $(s_{M + i}, s_{M + u})$, where $u$ is the number of exceeding indices.}
    The transition model is equivalent within the two groups, which is $p_i (s_H | s \in \mathcal{G}_1, a_1) = 1 - \Delta_2 = 1 - \log (H) / \sqrt{H}$ and $p_i (s_H | s \in \mathcal{G}_{2}, a_1) = 1 - \Delta_2 - \lambda / 2 = 1 - \log(H) / \sqrt{H} - \lambda / 2$. Thanks to the construction of $\mathcal{G}_1$ and $\mathcal{G}_2$, for every pair $\M_i, \M_j \in \M$ there is at least one state-action pair for which the $\lambda$-separation holds (Assumption~\ref{ass:separation_mdp}). To ease the visual inspection of the sample MDP $\M_i$ in Figure~\ref{fig:lb_instance_mdp}, the state $s_i$ and related transitions are blue colored, the states in $\mathcal{G}_1$ and related transitions are red colored.

    \paragraph{Event.} Together with the described instance, we use terminology from best policy identification (see Appendix~\ref{apx:best_policy_identification} for details) to define a convenient event around which the analysis is centered. We consider a class of stopping rules $\tau$ such that $\EV_{\M_i} [\tau] \leq H$, and we define:
    \begin{equation*}
        \mathcal{E} = \Big\{ \hat \pi_\tau \in \argmax\nolimits_{\pi \in \Pi} V_i(\pi) \ : \ \text{``best policy is identified within $H$ steps''} \Big\}.
    \end{equation*}
    To derive the lower bound, we consider the two cases in which $\mathcal{E}$ hold or does not hold with high probability, respectively.

    \paragraph{Bad Event.}
    If the event $\mathcal{E}$ does not hold with high probability, i.e., $\mathbb{P} (\mathcal{E}) < 1 - \delta$, then we can show that the regret scales with $\Omega (\sqrt{H})$.

    Let us consider any triplet $(\pi, \tau, \hat \pi_\tau)$. Without loss of generality, we take $\EV_{\M_i}[\tau] = h$, from which we have
    \begin{equation}
    \label{eq:regret_bad_event_mdp}
        \R_{H} (\M_i, \pi) = \R_{h} (\M_i, \pi) + \sum_{\bar h = h}^H V_{i}^* - V_{i} (\hat\pi_\tau) \geq \R_{h} (\M_i, \pi) + (H - h) \Delta_1
    \end{equation}
    with probability at least $\delta$. The latter inequality is obtained by noting that the set of policies going to the left in the initial state $\pi_1 (a_1 | s_{in}) = 1$, then taking action $a_1$ at some state $s_j \in (s_1, \ldots, s_M)$, and then taking the same action until the episode ends (formally, $\pi_t (a_2 | s_t) = 1$ for all $t < j$, $\pi_j (a_1 | s_j) = 1$, and $\pi_t (a_1 | s_H) = \pi_t (a_1 | s_L) = 1$ for all $t > j$) include an optimal policy for every $\M_i \in \M$. We denote this \emph{sufficient} set of policies as $\widetilde\Pi$. For any MDP $\M_i \in \M$, it holds $V_{i} (\pi) = \Delta_1$ for all $\pi \in \widetilde\Pi \setminus (\pi^*_i)$ and $V^*_i = 1$, which gives the above.

    Then, we lower bound the term $\R_{h} (\M_i, \pi)$ in \eqref{eq:regret_bad_event_mdp} through regret minimization. Due to how the instance is constructed, there is no incentive to take action $a_2$ in $s_{in}$ since the best policy identification fails in the bad event. Thus, we restrict the set of policies to $\widetilde\Pi$ again. Notably, this set of policies is finite, having size $|\widetilde\Pi| = M$. We can cast the regret minimization problem over this set of policies as a bandit with $M$ actions with parameters $(\mu_j = V_{i} (\pi_j))_{j \in [M]}$ for some arbitrary ordering of the policies in $\widetilde\Pi$. It is easy to see that the regret of the original MDP problem cannot be smaller than the regret of the latter bandit reformulation, which we can lower bound through the techniques in the proof of Theorem~\ref{thr:bandit_lower_bound}. We have
    \begin{equation}
    \label{eq:regret_reduced_to_bandit}
        \R_{h} (\M_i, \pi) \geq \frac{h \Delta_1}{8} \exp \left( - \frac{ h (\Delta_1)^2 }{2} \right).
    \end{equation}
    Finally, substituting~\eqref{eq:regret_reduced_to_bandit} into~\eqref{eq:regret_bad_event_mdp} we get
    \begin{equation*}
        \R_{H} (\M_i, \pi) \geq \frac{h \Delta_1}{8} \exp \left( - \frac{ h (\Delta_1)^2 }{2} \right) + (H - h) \Delta_1 \geq \frac{\sqrt{H}}{8} \exp \left( - \frac{1}{2} \right)
    \end{equation*}
    with probability at least $\delta$, where the last inequality is obtained by taking $\Delta_1 = 1 / \sqrt{H}$ and noting that the left-hand side is minimized for $\EV_{\M_i} [\tau] = h = H$. 
    
    \paragraph{Good Event.}
    The previous result states that the regret is at least $\Omega (\sqrt{H})$ when the event $\mathcal{E}$ does not hold with high probability. This hints that solving the best policy identification problem is necessary to minimize the regret. To derive the lower bound, we instantiate a proper best policy identification problem on the considered instance $\M$ and we derive the corresponding sample complexity through Lemma~\ref{lemma:best_policy}. We have $\EV_{\M_i} [\tau] \geq T^* (\M_i)^{-1} \log (1 / 2.4 \delta)$ where
    \begin{equation*}
        T^* (\M_i) = \sup_{\omega \in \Sigma (\M_i)} \inf_{\M_j \in \M_{-i}} \sum_{s, a} \omega (s, a) \KL_{\M_i | \M_j} (s,a).
    \end{equation*}
    From Assumption~\ref{ass:separation_mdp} and the Pinsker's inequality we have
    \begin{equation*}
        \KL_{\M_i | \M_j} (s,a) \geq 4 \text{TV}^2 (p_i(\cdot | s, a), p_j (\cdot | s, a)) \geq  \| (p_i - p_j)(\cdot | s, a) \|_1^2 \geq \lambda^2.
    \end{equation*}
    By staring at the instance, it can be seen that the allocation vector attaining the supremum is the one assigning even probabilities to all the pairs $(s_{M + x},a_1)_{x \in [M]}$, as it guarantees $ \omega(s,a) = 1 / TM$ for at least two revealing state-action pairs against any MDP $\M_j \in \M_{- i}$, such that $\sum_{s, a} \omega (s, a) \KL_{\M_i | \M_j} (s, a) \geq \frac{2\lambda^2}{TM}$, while any other allocation can be hacked by the infimum over $\M_j \in \M_{-i}$ to a lesser value.\footnote{There are actually other allocation vectors that have equivalent value of $T^* (\M_i)$, which is the one assigning even probabilities to all the pairs $(s, a_1)_{s \in \mathcal{G}_1}$ or $(s, a_1)_{s \in \mathcal{G}_2}$. For the sake of the proof, we use the most convenient to algebraic manipulations.}
    
    We just need to show that the desired allocation can be obtained and does not violate the flow constraints of the MDP (see the statement of Lemma~\ref{lemma:best_policy}). We set $\omega (s_1, a_2, 1) = 1$, which implies $\sum_a \omega (s_{M + x}, a) \leq 1 / T, \forall x \in [M]$, since the states of the right chain cannot be visited more than once in an episode. Then, we set $\omega (s_{M + x}, a_1) = 1 / TM, \forall x \in [M]$ from the desired allocation, which gives $\omega (s_{M + x}, a_1, x + 1) = 1 / M, \forall x \in [M]$. We have
    \begin{align*}
        \omega(s_{M + 1}, a_1, 2) = \frac{1}{M} \ \text{ and } \ &\omega(s_{M + 1}, a_2, 2) = \frac{M - 1}{M} \\ 
        &\text{from } \ \sum_a \omega (s_{M + 1}, a, 2) = \sum_{s', a'} p_i (s_{M + 1} | s', a') \omega (s', a', 1) = \omega (s_1, a_2, 1) = 1, \\
        \omega(s_{M + 2}, a_1, 3) = \frac{1}{M} \ \text{ and } \ &\omega(s_{M + 2}, a_2, 3) = \frac{M - 2}{M} \\ 
        &\text{from } \ \sum_a \omega (s_{M + 2}, a, 3) = \sum_{s', a'} p_i (s_{M + 2} | s', a') \omega (s', a', 2) = \omega (s_{M + 1}, a_2, 2) = \frac{M - 1}{M}, \\
        \omega(s_{M + 3}, a_1, 4) = \frac{1}{M} \ \text{ and } \ &\omega(s_{M + 3}, a_2, 4) = \frac{M - 3}{M} \\ 
        &\text{from } \ \sum_a \omega (s_{M + 3}, a, 4) = \sum_{s', a'} p_i (s_{M + 3} | s', a') \omega (s', a', 3) = \omega (s_{M + 2}, a_2, 3) = \frac{M - 2}{M}, \\
        \ldots \\
        \omega(s_{2M}, a_1, M + 1) = \frac{1}{M} \ \text{ and } \ &\omega(s_{2M}, a_2, M + 1) = 0 \\ 
        &\hspace{-0.4cm}\text{from } \ \sum_a \omega (s_{2M}, a, M + 1) = \sum_{s', a'} p_i (s_{2M} | s', a') \omega (s', a', M) = \omega (s_{2M - 1}, a_2, M) = \frac{1}{M},
    \end{align*}
    which gives $\omega (s_{M + x}, a_2, x) = \frac{M - x}{M}, \forall x \in [M]$, while all of the additional probability to have $\omega (s, a, t) \in \Prob (\S \times \A), \forall t \in [T]$ is absorbed by $s_L$ and $s_H$.
    
    Having proved that the desired allocation complies to the flow constraints, we proceed as
    \begin{equation*}
        \EV_{\M_i} [\tau] \geq \frac{TM}{2\lambda^2} \log \left( \frac{1}{2.4 \delta} \right).
    \end{equation*}
    Finally, we can derive the lower bound through
    \begin{align*}
        \R_H (\M_i, \pi) \geq \EV\nolimits_{\M_i} [\tau] \Delta_2 &= \frac{TM}{2\lambda^2} \frac{\log(H)}{ \sqrt{H}} \log \left( \frac{1}{2.4\delta} \right) \\
        &\geq \frac{1}{4 (\sqrt{C} - \log (C)) } \frac{TM \log (H)}{\lambda} \log \left( \frac{1}{2.4\delta} \right)
    \end{align*}
    where the last inequality is obtained by exploiting $H \leq C$ and that transition probabilities are in $[0, 1]$ to write $\frac{\lambda}{2} + \frac{\log (H)}{\sqrt{H}} \leq 1$, which gives $\lambda \leq \frac{ 2( \sqrt{C} - \log(C))}{\sqrt{H}}$.
\end{proof}

\subsection{Proofs of Section~\ref{sec:beyond_separation}}
\label{sec:apx_proofs_beyond_separation}

Here we report the proofs for the test-time regret upper bounds provided in Theorem~\ref{thr:mdp_upper_bound_clustering}, \ref{thr:mdp_upper_bound_tree}, \ref{thr:mdp_upper_bound_revealing_policies}, respectively.

\mdpClusteringUpperBound*
\begin{proof}
    The derivations are straightforward following the steps in the proof of Theorem~\ref{thr:mdp_upper_bound_separation} and being careful to count the expected number of episodes for both the ``Identify'' stages. 
    
    For the first ``Identify'' stage (lines 3-12 in Algorithm~\ref{alg:mdp_clustering}), we want to make sure that the cluster $\C^i_k$ to which the test task belongs, i.e., $\M_i \in \C_k^i$, is not eliminated from the set $\C$ with probability at least $1 - 1 / 2H$. We write
    \begin{align*}
        P \Big( &\text{``$\C_k^i$ is eliminated from $\C$''} \Big) \leq \sum_{u = 1}^{K - 1} P \Big( \text{``$\C_k^i$ is eliminated from $\C$ at iteration $u$''} \Big) 
    \end{align*}
    by noting that the loop is executed $K - 1$ times and through a union bound on the iterations. Next, we call Lemma~\ref{lemma:mdp_elimination} on the MDPs $\M_1, \M_2$ selected at line 4 (Algorithm~\ref{alg:mdp_clustering}) on the set $\X$ of samples collected at line 5 (Algorithm~\ref{alg:mdp_clustering}) to prove that the event ``$\C_k^i$ is eliminated from $\C$ at iteration $u$'' holds with probability less than $1 / 2KH$, and we further invoke Lemma~\ref{lemma:sampling_routine} to prove that $\X$ can be collected efficiently from the sampling routine (Algorithm~\ref{alg:mdp_sampling}).
    Let us denote $H_0$ the random variable associated to the number of episodes spent in the first ``Identify'' stage. We have
    \begin{equation}
    \label{eq:mdp_clustering_proof_1}
        \EV[H_0] = h_0 \leq 2 (K - 1) K n_{\C} \leq \frac{c_0 K^2 \log^2 (SKH / \lambda) \log (KH)}{\lambda^4}.
    \end{equation}

    Then, we look at the second ``Identify'' stage (lines 3-11 in Algorithm~\ref{alg:mdp_separation}), which is called at line 12 of Algorithm~\ref{alg:mdp_clustering} on the cluster $\hat \C$ resulting from the previous stage. Just like before, we want to make sure that the MDP $\M_i$ is not eliminated from the set $\hat \C$ with probability at least $1 - 1 / 2H$. We write
    \begin{align*}
        P \Big( &\text{``$\M_i$ is eliminated from $\hat\C$''} \Big) \leq \sum_{u = 1}^{N - 1} P \Big( \text{``$\M_i$ is eliminated from $\hat \C$ at iteration $u$''} \Big) 
    \end{align*}
    by noting that the loop is executed at most $N - 1$ times and through a union bound on the iterations. This follows verbatim the proof of Theorem~\ref{thr:mdp_upper_bound_separation} on the cluster $\hat \C$ insetad of $\M$, for which we can call Lemma~\ref{lemma:mdp_elimination} and Lemma~\ref{lemma:sampling_routine} to get similar results.
    We denote as $H_1$ the random variable associated to the number of episodes spent in the second ``Identify'' stage. We have
    \begin{equation}
    \label{eq:mdp_clustering_proof_2}
        \EV [H_1] = h_1 \leq 2 (N - 1) N n \leq \frac{c_1 N^2 \log^2 (SNH / \lambda) \log (NH)}{\lambda^4}.
    \end{equation}
    
    Then, we can call yet another union bound on the events defined for the two ``Identify'' stages and, with similar considerations as in Theorem~\ref{thr:mdp_upper_bound_separation}, we write
    \begin{align*}
        R_H (\M_i, \pi) &= \EV \left[ \sum_{h = 1}^{h_0 + h_1} V_i^* - V_i (\pi_h) \right] + \EV \left[ \sum_{h = h_0 + h_1}^H V_i^* - V_i (\hat \pi) \right] \\ 
        &\leq \frac{c_0 T K^2 \log^2 (SKH / \lambda) \log (KH)}{\lambda^4} + \frac{c_1 T N^2 \log^2 (SNH / \lambda) \log (NH)}{\lambda^4} \\
        &\leq \frac{c T (K^2 + N^2) \log^2 (SNH / \lambda) \log (NH)}{\lambda^4}
    \end{align*}
    where we plugged~\eqref{eq:mdp_clustering_proof_1} and~\eqref{eq:mdp_clustering_proof_2} together with $\EV [ \sum_{h = 1}^{h_0 + h_1} V_i^* - V_i (\pi_h) ] \leq \EV[H_0 + H_1] T$ to get the first inequality, $c = \max (c_0, c_1)$ and $N > K$ to obtain the second inequality.
\end{proof}

\mdpTreeUpperBound*
\begin{proof}
    The proof follows derivations of Theorem~\ref{thr:mdp_upper_bound_separation} with some slight yet crucial modifications. Just as in Theorem~\ref{thr:mdp_upper_bound_separation}, we want to show that the true MDP $\M_i$ will not be eliminated from $\D$ (lines 2-18 in Algorithm~\ref{alg:mdp_tree}) with probability at least $1 - 1/H$. However, the number of times the loop in lines 3-18 of Algorithm~\ref{alg:mdp_tree} is executed is not $M$, as in Theorem~\ref{thr:mdp_upper_bound_separation}, but depends on the depth of the tree structure given by Assumption~\ref{ass:mdp_tree}. Since the size of $\D$ is reduced at every iteration by a factor $\beta$ at least, we have
    \begin{equation*}
        M \beta^\text{depth} \leq 1 \quad \implies \quad \text{depth} \leq \frac{\log (1 / M)}{\log (\beta)} = \log_{1 / \beta} (M) =: d.
    \end{equation*}
    With the latter, we can write
    \begin{equation*}
        P \Big( \text{``$\M_i$ is eliminated from $\D$''} \Big) 
        \leq \sum_{m = 1}^{d} P \Big( \text{``$\M_i$ is eliminated from $\D$ at iteration $m$''} \Big) 
    \end{equation*}
    through a union bound on the iterations. Then, we call Lemma~\ref{lemma:mdp_elimination} (with care of substituting $M$ with $d$ the result holds verbatim) to prove that each event ``$\M_i$ is eliminated from $\D$ at iteration $m$'' in the summation holds with probability less than $1 / d H$.

    For each iteration of the outer loop (lines 3-18), we want to prove that the $n$ samples from $(\bar s, \bar a)$ can be collected efficiently through the inner loop (lines 6-12). We follow similar steps as in Lemma~\ref{lemma:sampling_routine}. Through Assumption~\ref{ass:strong_reachability}, we have that $\EV [X (\bar s, \bar a | \M_i, \pi)] \leq T / 2$. By applying the Markov's inequality we get $X (\bar s, \bar a | \M_i, \pi) \leq T$ with probability at least $1 / 2$. Let $Y$ the random variable denoting the number of episodes needed to reach state $(\bar s, \bar a)$ in the test MDP $\M_i$. We have that $P(Y = k) \leq 1 / 2^k$, which gives $\EV[Y] \leq \sum_{k = 1}^{\infty} k / 2^k = 2$. Thus, the expected number of episodes to collect the $n$ samples is upper bounded by $2n$.

    Let $H_0$ denote the random variable associated to the number of episodes spent in the ``Identify'' stage as a whole (lines 2-18). From the considerations above, we have 
    \begin{equation}
    \label{eq:mdp_tree_proof_1}
        \EV [H_0] = h_0 \leq 2 d n \leq \frac{c d \log^2 (S d H / \lambda) \log (d H)}{\lambda^4}.
    \end{equation}
    Then, we can write
    \begin{align*}
        R_H (\M_i, \pi) &= \EV \left[ \sum_{h = 1}^{h_0} V_i^* - V_i (\pi_h) \right] + \EV \left[ \sum_{h = h_0}^H V_i^* - V_i (\hat \pi) \right] \leq \frac{c T d \log^2 \left(\frac{S d H}{\lambda}\right) \log (d H)}{\lambda^4}
    \end{align*}
    where we plugged~\eqref{eq:mdp_tree_proof_1} together with $\EV [ \sum_{h = 1}^{h_0} V_i^* - V_i (\pi_h) ] \leq \EV[H_0] T$ to get the inequality.
\end{proof}

\mdpRevealingUpperBound*
\begin{proof}
    The proof follows the derivation of Theorem~\ref{thr:mdp_upper_bound_separation} verbatim, with the only difference that the samples $\X$ for each iteration of the loop in the ``Identify'' stage (lines 13-21 of Algorithm~\ref{alg:mdp_revealing_policies}) are entirely collected during the ``Explore'' stage (lines 3-11). Analogously as before, we show that the true MDP $\M_i$ will not be eliminated from $\D$ (lines 6-10 in Algorithm~\ref{alg:mdp_separation}) with probability at least $1 - 1 / H$. Especially, we can write
    \begin{align*}
        P \Big( &\text{``$\M_i$ is eliminated from $\D$''} \Big) 
        \leq \sum_{m = 1}^{M - 1} P \Big( \text{``$\M_i$ is eliminated from $\D$ at iteration $m$''} \Big) 
    \end{align*}
    by noting that the loop in the ``Identify'' stage is executed for $M - 1$ iterations and then applying a union bound. Just like before, we call Lemma~\ref{lemma:mdp_elimination} to prove that the event ``$\M_i$ is eliminated from $\D$ at iteration $m$'' holds with probability less than $1 / MH$, and then we invoke Lemma~\ref{lemma:sampling_routine} to prove that the ``Explore'' stage can collect the desired number of samples efficiently.
    Let $H_0$ denote the random variable associated to the number of episodes spent in the ``Explore'' stage. Following similar considerations as in the proof of Theorem~\ref{thr:mdp_upper_bound_separation}, we have 
    \begin{equation}
    \label{eq:mdp_revealing_proof_1}
        \EV [H_0] = h_0 \leq 2 I n \leq \frac{c I \log^2 (SMH / \lambda) \log (MH)}{\lambda^4}.
    \end{equation}
    Then, we can write
    \begin{align*}
        R_H (\M_i, \pi) &= \EV \left[ \sum_{h = 1}^{h_0} V_i^* - V_i (\pi_h) \right] + \EV \left[ \sum_{h = h_0}^H V_i^* - V_i (\hat \pi) \right] \leq \frac{c T I \log^2 \left(\frac{SMH}{\lambda}\right) \log (MH)}{\lambda^4}
    \end{align*}
    where we plugged~\eqref{eq:mdp_revealing_proof_1} together with $\EV [ \sum_{h = 1}^{h_0} V_i^* - V_i (\pi_h) ] \leq \EV[H_0] T$ to get the inequality.
\end{proof}

\section{Best Policy Identification in Finite-Horizon MDPs: A Tailored Lower Bound}
\label{apx:best_policy_identification}

In Best Policy Identification~\citep[BPI,][]{fiechter1994efficient}, the learner interacts with an unknown MDP $\M_i$ with the goal of minimizing the expected number of samples to be taken in order to tell an optimal policy $\pi^* \in \argmax_{\pi \in \Pi} V_{i} (\pi)$ for $\M_i$ with probability at least $1 - \delta$, where $\delta \in (0, 1)$ is a fixed confidence. 

The literature provides theoretical guarantees on the latter expected number of samples, called \emph{sample complexity}, in a variety of settings ranging from \emph{worst-case} results for discounted~\citep{gheshlaghi2013minimax, agarwal2020model} and finite-horizon MDPs~\cite{dann2015sample, dann2019policy, kaufmann2021adaptive, menard2021fast} to \emph{instance-dependent} analyses~\citep{al2021adaptive, al2021navigating, wagenmaker2022beyond, tirinzoni2022near, tirinzoni2023optimistic, al2023towards}.

For the purpose of deriving a lower bound for test-time regret minimization (Theorem~\ref{thr:mdp_lower_bound}), we use, as a building block, an instance-dependent, non-asymptotic lower bound to the sample complexity of any $\delta$-PC (Probably Correct) BPI algorithm in finite-horizon MDPs.\footnote{A $\delta$-PC algorithm~\citep[e.g.,][]{al2021navigating} is an algorithm that is guaranteed to output an optimal policy $\pi^*$ with probability at least $1 - \delta$ with a finite sample complexity.} To the best of our knowledge, the only result of this kind is given in~\citet[][Theorem 2]{al2023towards}. Here we derive an alternative result that is tailored to our setting, i.e., in which the set of possible MDPs is restricted to a finite set $\M$ fulfilling the $\lambda$-separation (Assumption~\ref{ass:separation_mdp}).

\paragraph{Additional Notation.}
Let $\Hist_h := (s_t, a_t, r_t)_{t \in [T]}$ be a trajectory collected by executing a policy $\pi_h$ at the episode $h$. We denote $\F_h := \sigma((\Hist_{h'})_{h' \in [h]})$ the sigma algebra of the trajectories up to episode $h$, such that $(\F_{h'})_{h' \in [h]}$ is the corresponding filtration. We define
\begin{itemize}
    \item $(\pi_{h'} : \F_{h' - 1} \to \Prob(\A))_{h' \in [h]}$ a \emph{sampling rule} that determines the policy to be run at each episode given the past observations;
    \item $\tau$ a \emph{stopping rule} that gives the time at which the sampling process is stopped given past observations;
    \item $\hat \pi_\tau \in \Pi$ a \emph{decision rule}, which is the policy selected when $\tau$ is triggered, i.e., the best guess on the optimal policy given past observations.
\end{itemize}
We denote as $\EV [\tau]$ the sample complexity of the BPI problem. Notably, the identification can span several episodes of our finite-horizon MDP setting, which means that at any step $h'$ such that $mod (k', T)  = 0$ the process will be reset to state $s_1$.
To simplify the analysis, we assume that whenever the stopping rule $\tau$ is triggered, the process proceeds until the end of the episode, which means the sample complexity is a multiple of $T$.

Now, we have all of the elements to derive our tailored lower bound. Specifically, we adapt to our BPI problem of interest the result~\citep[][Proposition 2]{al2021navigating}, which was originally derived for the infinite-horizon and $\delta$-asymptotic setting. We obtain the following.

\begin{lemma}[Best policy identification]
\label{lemma:best_policy}
    Let assume all the $\M_j \in \M$ admit unique optimal policies. For $\M_i \in \M$, let us define the set of allocation vectors
    \begin{align*}
        \Sigma (\M_i) = \Big\{ \omega \in \Prob(\S \times \A) :\quad &\omega (s, a) =  \frac{1}{T} \sum_{t \in [T]} \omega (s, a, t), \\ &\omega (\cdot, \cdot, 1) \in \Prob (\S \times \A), \quad \sum_{a \in \A} \omega (s_1, a, 1) = 1, \\
        &\forall (s, t) \in \S \times (2, \ldots, T), \quad \sum_{a \in \A} \omega (s, a, t) = \sum_{s', a'} p_{i} (s | s', a') \omega (s', a', t - 1)  \Big\}.
    \end{align*}
    Let $\M_{-i} := \M \subseteq \M_i$. For $\delta \in (0, 1)$, any $\delta$-PC BPI algorithm has sample complexity
    \begin{equation*}
        \EV_{\M_i} [\tau] \geq T^* (\M_i)^{-1} \log (1 / 2.4 \delta) \quad \text{ where } \quad T^* (\M_i) = \sup_{\omega \in \Sigma (\M_i)} \inf_{\M_j \in \M_{-i}} \sum_{s, a} \omega (s, a) \KL_{\M_i | \M_j} (s,a).
    \end{equation*}
\end{lemma}
\begin{proof}
    The derivations are adapted from the proof of Proposition 2 in~\citet{al2021navigating}. First, we report a sample complexity result on best policy identification with a generative model~\citep{al2021adaptive}, which, for any $\M_j \in \M_{-i}$, states that
    \begin{equation}
        \sum_{s, a} \EV_{\M_i} [N_{\tau} (s, a)] \KL_{\M_i | \M_j} (s,a) \geq \text{kl} (\delta, 1 - \delta)
        \label{eq:generative_model_sample_complexity}
    \end{equation}
    where $N_{\tau} (s, a)$ is the number of samples of the $(s, a)$ pair collected within $\tau \in \mathbb{N}$ steps and $\text{kl} (x, y)$ denotes the Kullback Leibler divergence between Bernoulli distributions with parameters $x, y$ respectively. Differently from the generative model setting, we have to enforce MDP constraints on $\EV_{\M_i} [N_{\tau} (s, a)]$, which gives the recursive expression
    \begin{align*}
        &\text{if } mod(\tau, T) \neq 0
        & \sum_{a} \EV_{\M_i} [ N_\tau (s, a)] = \sum_{s', a'} p_{\M_i} (s | s', a')  \left(\EV_{\M_i} [N_{\tau - 1} (s', a')] + 1 \right), \quad \forall (s, a) \in \S \times \A \\
        &\text{else}
        &\sum_{a} \EV_{\M_i} [ N_\tau (s_1, a) ] = \sum_{a}\EV_{\M_i} [ N_{\tau - 1} (s_1, a) ] + 1 \quad \text{and} \quad \EV_{\M_i} [ N_\tau (s, a) ] = \EV_{\M_i} [ N_{\tau - 1}  (s, a)], \quad \forall s \neq s_1.
    \end{align*}
    Then, we can combine the latter constraints with \eqref{eq:generative_model_sample_complexity} to write the following optimization problem
    \begin{align}
        &\inf_{n \geq 0} \sum_{s, a, t} n_{sat} \label{eq:prob_obj} \\
        &\text{subject to } \sum_{s,a,t} n_{sat} \KL_{\M_i | \M_j} (s, a) \geq \text{kl} (\delta, 1 - \delta)
        & \forall \M_j \in \M_{-i} \label{eq:prob_cons1} \\
        &\hspace{1.7cm} \sum_{a} n_{sat} = \sum_{s',a'} p_{\M_i} (s | s', a') (n_{s'a't-1} + 1)
        & \forall s, \forall t : mod(t, T) \neq 0 \label{eq:prob_cons2} \\
        &\hspace{1.7cm} \sum_a n_{s_1 at} = \sum_a n_{s_1 at-1} + 1
        & \forall t : mod(t, T) = 0 \label{eq:prob_cons3} \\
        &\hspace{1.7cm} n_{sat} = n_{sat-1}
        & \forall s \neq s_1, \forall t : mod(t, T) = 0 \label{eq:prob_cons4}
    \end{align}
    To prove the result, let us take the constraint~\eqref{eq:prob_cons1}. Since it has to hold for every $\M_j \in \M_{-i}$, we can write
    \begin{equation*}
        \inf_{\M_j \in \M_{-i}} \sum_{s,a,t} n_{sat} \KL_{\M_i | \M_j} (s, a) \geq \text{kl} (\delta, 1 - \delta).
    \end{equation*}
    Let us denote $N^*$ the value of~\eqref{eq:prob_obj}, we write
    \begin{equation*}
        \inf_{\M_j \in \M_{-i}} \sum_{s,a,t} \frac{n_{sat}}{N^*} \KL_{\M_i | \M_j} (s, a) \geq \frac{\text{kl} (\delta, 1 - \delta)}{N^*}.
    \end{equation*}
    Through constraints (\ref{eq:prob_cons2}-\ref{eq:prob_cons4}) and the definition of $N^*$, we have that $(n_{sat} / N^*) \in \Sigma(\M_i)$, where the latter is the set of allocation vectors.
    Hence, we can write
    \begin{equation*}
        \sup_{\omega \in \Sigma(\M_i)} \inf_{\M_j \in \M_{-i}} \sum_{s,a} \omega (s,a) \KL_{\M_i | \M_j} (s, a) \geq \frac{\text{kl} (\delta, 1 - \delta)}{N^*}.
    \end{equation*}
    Finally, we know from~\citep[][Proposition 10]{al2021navigating} that $\EV_{\M_i} [\tau] \geq N^*$, which together with $\text{kl} (\delta, 1 - \delta) \geq \log(1 / 2.4 \delta)$~\citep[see][]{kaufmann2016complexity} gives the result.
\end{proof}

\section{Meta Learning in Bandits}
\label{apx:bayesian_bandits}

In this section, we analyze a simplified bandit version of the test-time regret minimization problem described in the paper. The aim of this study is to serve both as a gentle introduction to the more advanced results and techniques presented in the paper, which come more naturally in the bandit setting, as well as a standalone analysis that may be of independent interest.

We consider a class $\M = (\M_i)_{i \in [M]}$ of bandits~\cite{lattimore2020bandit}, each of them having a set of actions $\A = (a_j)_{j \in [A]}$ with corresponding reward distributions $R_i (a_j)$ for all $i \in [M], j \in [A]$ with bounded mean $\mu_{ij} \in [0, 1]$.
First, we rephrase the separation condition presented in the paper (Assumption~\ref{ass:separation_mdp}) as follows,
\begin{assumption}[$\lambda$-separation (bandit)]
\label{ass:separation_bandit}
    For any $\M_i, \M_j \in \M$, there exists $a \in \A$ such that $\| R_i (a) - R_j (a) \|_1 \geq \lambda.$
\end{assumption}

Just like in the MDP setting, we assume the reward distribution of all bandits $\M_i \in \M$, as well as the set $\M$ itself, to be fully known to the agent, who faces a test task (i.e., bandit) that is instead unknown but belonging to $\M$. To evaluate the agent's performance, we redefine the $H$-steps test-time regret for the task $\M_i \in \M$ as
\begin{equation*}
    \R_H (\M_i, \pi) = \EV \left[ \sum_{h = 1}^H R_i(a^*) - R_i(a_h) \right] = \EV \left[ \sum_{h = 1}^H \mu^* - \mu_h \right] 
\end{equation*}
where $a^* \in \argmax_{a \in \A} \mu (a)$ is the optimal action in the bandit $\M_i$, $a_h \in \A$ is the action played by policy $\pi$ at step $h$, and $\mu^*, \mu_h$ are the mean of their reward distribution, respectively.
Just as we did in the paper for the more general MDP setting, we provide both a lower bound and a nearly minimax optimal algorithm for the test-time regret minimization in bandits.

\subsection{Lower Bound}

We now prove a lower bound to the test-time regret suffered by any algorithm in the introduced meta learning in bandits setting under the above separation condition (Assumption~\ref{ass:separation_bandit}). Formally,
\begin{theorem}[Lower bound]
    Let $\M$ be a set of bandits for which Assumption~\ref{ass:separation_bandit} holds. Let $C < \infty$ a constant, and let $\delta \in (0, 1)$. For any horizon $ M - 1 \leq H \leq C$, it holds
    \begin{equation*}
        R_H (\M_i, \pi) = \Omega \left( \frac{M \log (H)}{\lambda} \log \left( \frac{1}{\delta} \right) \right)
    \end{equation*}
    with probability at least $1 - \delta$.
    \label{thr:bandit_lower_bound}
\end{theorem}

\begin{proof}
    To prove the lower bound, we first construct a convenient instance in which it is hard to minimize the regret without knowing the task, while it is easy to identify the task playing sub-optimal actions. Then, we derive the lower bound on the regret suffered by any algorithm leveraging minimax lower bounds for standard bandits~\citep{lattimore2020bandit} and best arm identification results~\citep{garivier2016optimal}.
    
    Let $\M = (\M_i)_{i \in [M]}$ a problem instance in which every $\M_i$ has $|\A| = 2M$ actions and Gaussian reward distributions $R_i (a_j) = \mathcal{N} (\mu_{ij}, 1)$. For each $\M_i$, we specify the first set of $M$ actions $(a_j)_{j = 1}^{M}$ as follows: The action $a_i$ is the optimal action with mean $\mu_i = \mu^*$, while all of the other actions are slightly sub-optimal $\mu^* - \mu_j = \Delta_1 = 1 / \sqrt{H}$. The second set of actions $(a_j)_{j = M + 1}^{2M}$ is specified as follows: The actions $\mathcal{A}_1 := (a_{M + i}, \ldots, a_{\frac{3M}{2} + i})$ have mean reward  such that $\mu^* - \mu_a = \Delta_2 = \log(H) / \sqrt{H}, \forall a \in \A_1$, and all of the other actions $\mathcal{A}_2 := (a_{M + 1}, \ldots, a_{M + i - 1}) \cup (a_{\frac{3M}{2} + i + 1}, \ldots, 2M)$ have mean reward $\mu_a < \mu_{M + i}$ such that $\| R(a_{M + i}) - R (a) \|_1 \geq \lambda, \forall a \in \A_2$, fulfilling $\lambda$-separation. The instance is depicted in Figure~\ref{fig:lb_instance}.\footnote{Note that, for describing the instance, we conveniently consider $i \leq M / 2$, but it is straightforward to understand how it works for $i > M/2$ by substituting the exceeding indices in $\A_1$ back with the first arms of the sequence $a_{M + 1}, a_{M + 2}, \ldots$ until $\A_1$ consists of $M / 2$ arms.}

    In general, there are two ways to approach the described instance. Since it is known that the second set of $M$ actions is sub-optimal in every $\M_i$, we can minimize the regret playing only the first set of actions. Otherwise, we can exploit the separation condition on the second set of arms to identify the task and then playing the optimal arm. To formalize this intuition, we borrow notation from best arm identification literature~\citep[e.g.][]{garivier2016optimal} similarly as we did in Appendix~\ref{apx:best_policy_identification}.

    \paragraph{Additional Notation.}
    Let $\Hist_h := (a_{h'}, r_{h'})_{h' \in [h]}$ be a trajectory collected by executing a policy $\pi$. We denote $\F_{h \geq 1}$ the corresponding filtration on $\Hist_h$. We define
    \begin{itemize}
        \item $(\pi_{h'})_{h' = 1}^h$ a \emph{sampling rule} over $\A$ that determines the next action to play given past observations;
        \item $\tau$ a \emph{stopping rule} that gives the stopping time w.r.t. $\mathcal{H}_h$;
        \item $\hat a_\tau \in \A$ a \emph{decision rule}, which is the action selected when $\tau$ is triggered, i.e., the best guess on the optimal arm given past observations.
    \end{itemize}
    We denote as $\EV [\tau]$ the sample complexity of the best arm identification problem.
    Further, we restrict $\tau$ to the class of stopping rules such that $\EV_{\M_i} [\tau] \leq H$, and we define the following event:
    \begin{equation*}
        \mathcal{E} = \Big\{ \hat a_\tau \in \argmax\nolimits_{a \in \A} \mu(a) \ : \ \text{``best arm is identified within $H$ steps''} \Big\}.
    \end{equation*}
    To derive the lower bound, we consider the two cases in which $\mathcal{E}$ hold or does not hold with high probability, respectively.

    \begin{figure}[t]
        \centering
        \includegraphics[width=0.75\textwidth]{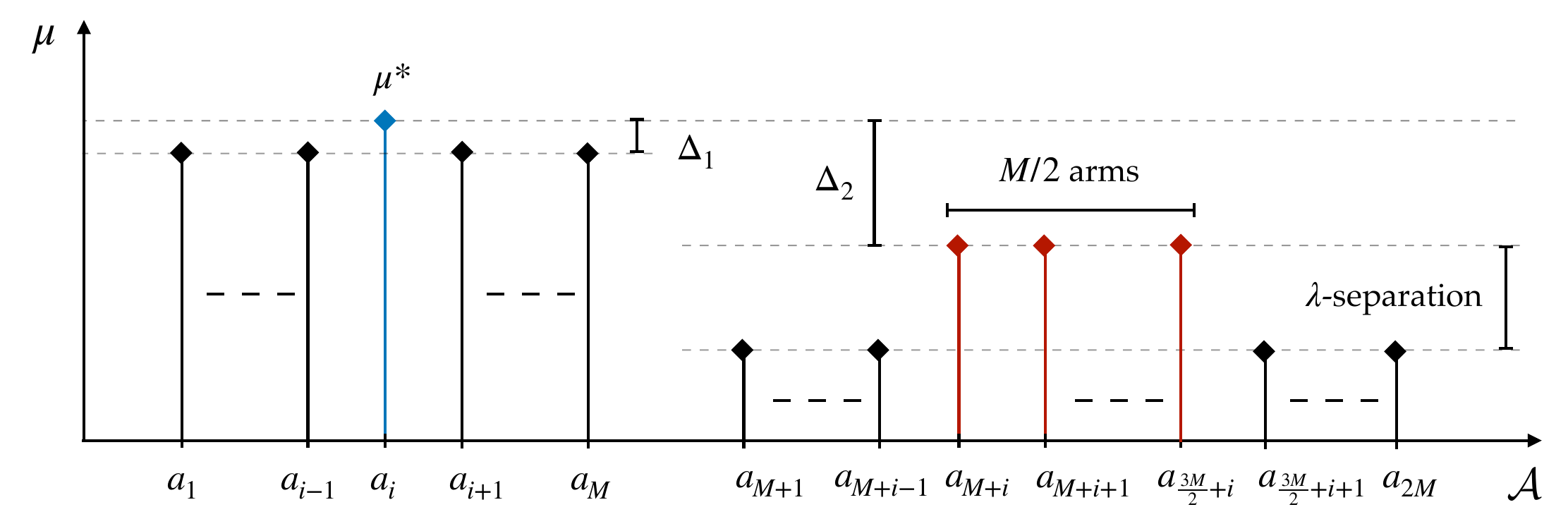}
        \caption{Visualization of the $\M_i$ bandit in the problem instance designed to derive the lower bound. The optimal action $a_i$ and the identifying actions $a \in \A_1 \cup \A_2$ change for every $\M_i$.}
        \label{fig:lb_instance}
    \end{figure}

    \paragraph{Bad Event.} If the event $\mathcal{E}$ does not hold with high probability, i.e., $\mathbb{P} (\mathcal{E}) < 1 - \delta$, we can show that the regret scales with $\Omega (\sqrt{H})$.

    Let us consider any triplet $(\pi, \tau, \hat a_\tau)$. Without loss of generality, we take $\EV_{\M_i}[\tau] = h$, from which we have
    \begin{equation}
    \label{eq:regret_bad_event}
        \R_{H} (\M_i, \pi) = \R_{h} (\M_i, \pi) + (H - h) \Delta_1
    \end{equation}
    with probability at least $\delta$.

    First, we lower bound the term $\R_{h} (\M_i, \pi)$ through regret minimization.
    We can restrict the action set to $\tilde\A = (a_j)_{j = 1}^M$ as there is no incentive to play surely sub-optimal actions $a_j$ for $j > M$ when minimizing the regret. We take a policy $\pi$ inducing pulls $(a_{h'})_{h' = 1}^h$ and corresponding counts $T_j (h) := \sum_{h' = 1}^h \mathbf{1} (a_j = a_{h'})$ over the actions $\tilde\A$ of $\M_1$. Then, we select $\M_i \in \M$ such that $i = \argmin_{j \in [M]} \EV_{\M_1} [T_j (h)]$. We have that
    \begin{equation*}
        \max_{\M^* \in \M} \R_h (\M^*, \pi) \geq \max \left\{ \R_h (\M_1, \pi), \R_h (\M_i, \pi) \right\} \geq \frac{\R_h (\M_1, \pi) + \R_h (\M_i, \pi)}{2}.
    \end{equation*}
    We can further expand the terms on the right hand-side by noting that
    \begin{equation*}
        \R_h (\M_1, \pi) \geq P_{\M_1} (T_1 (h) \leq h / 2) \frac{h \Delta_1}{2} \ \ \text{and} \ \ \R_h (\M_i, \pi) \geq P_{\M_i} (T_1 (h) > h / 2) \frac{h \Delta_1}{2}
    \end{equation*}
    from which we can write
    \begin{align*}
        \R_h (\M_1, \pi) + \R_h (\M_i, \pi) &> \frac{h \Delta_1}{2} \left( P_{\M_1} (T_1 (h) \leq h / 2) + P_{\M_i} (T_1 (h) > h / 2) \right) \\
        &\geq \frac{h \Delta_1}{4} \exp \left( - \text{KL} (\mathbb{P}_{\M_1}, \mathbb{P}_{\M_i}) \right)
    \end{align*}
    where the latter is obtained through the Bretagnolle-Huber inequality. Then, we can upper bound the KL divergence as
    \begin{align*}
        \text{KL} & (\mathbb{P}_{\M_1}, \mathbb{P}_{\M_i}) \\
        &= \EV_{\M_1} [T_1 (h)] \text{KL} (\mathcal{N} (\mu^*, 1), \mathcal{N} (\mu - \Delta_1, 1)) + \EV_{\M_1} [T_i (h)] \text{KL} (\mathcal{N} (\mu - \Delta_1, 1), \mathcal{N} (\mu^*, 1)) \\
        &\leq \frac{(\Delta_1)^2}{2} \left( \EV_{\M_1} [T_1 (h)] + \EV_{\M_1} [T_i (h)] \right) \leq \frac{ h (\Delta_1)^2}{2}
    \end{align*}
    from which we derive
    \begin{equation*}
        \max_{\M^* \in \M} \R_h (\M^*, \pi) 
        \geq \frac{h \Delta_1}{8} \exp \left( - \frac{ h (\Delta_1)^2 }{2} \right).
    \end{equation*}
    Finally, we substitute the latter in~\eqref{eq:regret_bad_event} to get
    \begin{equation*}
        \R_{H} (\M_i, \pi) \geq \frac{h \Delta_1}{8} \exp \left( - \frac{ h (\Delta_1)^2 }{2} \right) + (H - h) \Delta_1 \geq \frac{\sqrt{H}}{8} \exp \left( - \frac{1}{2} \right)
    \end{equation*}
    with probability at least $\delta$, where the last inequality is obtained by taking $\Delta_1 = 1 / \sqrt{H}$ and noting that the left-hand side is minimized for $\EV_{\M_i} [\tau] = h = H$. 

    \paragraph{Good Event.} The previous result states that the regret is at least $\Omega (\sqrt{H})$ when the event $\mathcal{E}$ does not hold with high probability. This hints that solving the best arm identification problem is necessary to minimize the regret. To derive the lower bound, we instantiate a proper best arm identification problem on the considered instance $\M$. 
    
    Since the separation condition is fulfilled in the second set of actions, we can restrict our best arm identification problem to the action set $\hat\A = (a_j)_{j = M + 1}^{2M}$. From Theorem 1 in~\citep{garivier2016optimal}, for any confidence $\delta \in (0, 1)$, we have that
    \begin{equation*}
        \EV\nolimits_{\M_i} [\tau] \geq T^* (\M_i)^{-1} \log (1 / 2.4 \delta)
    \end{equation*}
    where
    \begin{equation}
    \label{eq:t_star}
        T^* (\M_i)^{-1} := \sup_{\omega \in \Prob (\hat\A) } \inf_{\M_j \in \M \setminus (\M_i) } \left( \sum_{k = M + 1}^{2M} \omega_k \text{KL} (R_i(a_k), R_j (a_k)) \right)
    \end{equation}
    holds with probability $1 - \delta$.
    From the separation condition (Assumption~\ref{ass:separation_bandit}) and the Pinsker's inequality we have
    \begin{equation*}
        \text{KL} (R_i(a), R_j (a)) \geq 2 \text{TV}^2 (R_i(a), R_j (a)) \geq \frac{1}{2} \| R_i(a) - R_j (a) \|_1^2 \geq \frac{\lambda^2}{2}
    \end{equation*}
    for every $i \neq j$. By noting that the supremum in~\eqref{eq:t_star} is attained by $\omega = (1 / M, \ldots, 1 / M)$ we get
    \begin{equation*}
        \EV\nolimits_{\M_i} [\tau] \geq \frac{2M}{\lambda^2} \log \left( \frac{1}{2.4\delta} \right).
    \end{equation*}

    Finally, we can derive the lower bound through
    \begin{equation*}
        \R_H (\M_i, \pi) \geq \EV\nolimits_{\M_i} [\tau] \Delta_2 = \frac{2M}{\lambda^2} \frac{\log(H)}{ \sqrt{H}} \log \left( \frac{1}{2.4\delta} \right) \geq \frac{2}{\sqrt{C} - \log (C) } \frac{M \log (H)}{\lambda} \log \left( \frac{1}{2.4\delta} \right)
    \end{equation*}
    where the last inequality is obtained by exploiting $H \leq C$ and that the mean of the reward distribution is bounded in $[0, 1]$ to write $\lambda + \frac{\log (H)}{\sqrt{H}} \leq 1$, which gives $\lambda \leq \frac{\sqrt{C} - \log(C)}{\sqrt{H}}$.
\end{proof}

\subsection{Upper Bound}

In this section, we provide a simple algorithm, which is practically a direct adaptation to the bandit setting of Algorithm~\ref{alg:mdp_separation}, in turn inspired by~\citet{chen2021understanding}, that nearly matches the lower bound presented in the previous section.

The idea of the algorithm is to exploit the knowledge of the class $\M$ to quickly identify the test task $\M^*$ and then commit to the optimal action $a^*$ for $\M^*$. The pseudocode of this simple procedure is provided in Algorithm~\ref{alg:bandit}.

\begin{algorithm}[H]
    \caption{Identify-then-Commit for Bandits}
    \label{alg:bandit}
    \begin{algorithmic}[1]
        \STATE Initialize $\D = \M$ and $n = \frac{2 \log (2 M H)}{\lambda^4} $
        \WHILE{$|\D| > 1$}
            \STATE Draw $\M_1, \M_2$ from $\D$ randomly
            \STATE Take $\tilde a \in \argmax_{a \in \A} \| R_1 (a) - R_2(a) \|_1$
            \STATE Collect $n$ samples $\X = (x_1, \ldots, x_n)$ pulling action $\tilde a$
            \IF{$\prod_{x_h \in \X} \frac{R_1 (x_h | \tilde a)}{R_2 (x_h | \tilde a)} \geq 1$}
                \STATE Eliminate $\M_2$ from $\D$
            \ELSE
                \STATE Eliminate $\M_1$ from $\D$
            \ENDIF
        \ENDWHILE
        \STATE Take $\hat \M \in \D$ and pull the action $\hat a \in \argmax_{a \in \A} \hat \mu (a)$ for the remaining steps
    \end{algorithmic}
\end{algorithm}

The upper bound to the test-time regret suffered by Algorithm~\ref{alg:bandit} is provided by the following result.
\begin{theorem}[Upper bound]
    Let $\M$ be a set of bandits for which Assumption~\ref{ass:separation_bandit} holds. For any $H \geq M - 1$, we have
    \begin{equation*}
        R_H (\M_i, \hat \pi) = \cO \left( \frac{M \log (MH)}{\lambda^4} \right)
    \end{equation*}
    where $\hat \pi$ is the sampling rule induced by Algorithm~\ref{alg:bandit}.
    \label{thr:bandit_upper_bound}
\end{theorem}
\begin{proof}
    The scheme of the proof follows closely the one of~\citet[][Theorem 1]{chen2021understanding}, which is simplified and adapted to the bandit setting we care about here.

    First, we note that the Algorithm~\ref{alg:bandit} is made of two stages: An ``Identify'' stage (lines 2-11) in which we seek to find out the test task $\M_i$ irrespective of the regret, a ``Commit'' stage (line 12) in which we exploit the gathered information to minimize the regret in the remaining steps.
    Notably, at every iteration of the while loop (lines 2-11) a potential task is eliminated from the set $|\D|$, which means the ``Identify'' stage consists of exactly $h_0 := (M - 1) n$ steps, and the ``Commit'' stage takes the remaining $H - h_0$ steps.
    Thus, we can decompose the regret as
    \begin{equation}
        R_H (\M_i, \hat \pi) = \EV \left[ \sum_{h = 1}^{h_0} R_i(a^*) - R_i (\tilde a) \right] + \EV \left[ \sum_{h = h_0}^H R_i(a^*) - R_i (\hat a) \right].
        \label{eq:bandit_upper_bound_decomposition}
    \end{equation}
    Now, we just need to upper bound the term on the left with $h_0$ through $R_i(a^*) - R_i (\tilde a) \leq 1$ and to show that the second term is zero with high probability to prove the result.

    Since $\hat a$ is the optimal action of the remaining task in the set $\D$, to prove that it holds $\EV [\sum_{h = h_0}^H R_i(a^*) - R_i (\hat a) ] = 0$ with high probability, we have to show that the test task $\M_i$ is not eliminated from $\D$ with high probability. Especially, for some confidence $\delta \in (0, 1)$ we need
    \begin{equation*}
        \mathbb{P} \Big( \text{``$M_i$ is eliminated from $\D$''} \Big) = \mathbb{P} \left( \prod_{x_h \in \X} \frac{R_i (x_h | \tilde a)}{R_j (x_h | \tilde a)} < 1 \right) \leq \frac{\delta}{M}
    \end{equation*}
    where the right-hand side is obtained from a union bound over the event that the test task $\M_i$ is eliminated in each iteration of the while loop (lines 2-11). Equivalently, we need
    \begin{equation*}
        \log \left( \prod_{x_h \in \X} \frac{R_i (x_h | \tilde a)}{R_j (x_h | \tilde a)} \right) = \sum_{x_h \in \X} \log \left( \frac{R_i (x_h | \tilde a)}{R_j (x_h | \tilde a)} \right) > 0
    \end{equation*}
    to hold with probability at least $1 - \frac{\delta}{M}$.
    First, we note that
    \begin{align*}
        \EV_{x \sim R_i (\tilde a)} \left[ \sum_{x_h \in \X} \log \left( \frac{R_i (x | \tilde a)}{R_j (x | \tilde a)} \right) \right]
        = \sum_{x_h \in \X}  \EV_{x \sim R_i (\tilde a)} \left[ \log \left( \frac{R_i (x | \tilde a)}{R_j (x | \tilde a)} \right) \right] = n \text{KL} (R_i (\tilde a), R_j (\tilde a)) \leq \frac{n \lambda^2}{2}
    \end{align*}
    where the last inequality is obtained from the separation condition (Assumption~\ref{ass:separation_bandit}) and the Pinsker's inequality.
    Then, we have
    \begin{align*}
        \sum_{x_h \in \X} \log \left( \frac{R_i (x_h | \tilde a)}{R_j (x_h | \tilde a)} \right) 
        \geq \EV_{x \sim R_i (\tilde a)} \left[ \sum_{x_h \in \X} \log \left( \frac{R_i (x | \tilde a)}{R_j (x | \tilde a)} \right) \right] - \sqrt{\frac{n}{2} \log \left( \frac{2M}{\delta} \right)} \geq  \frac{n \lambda^2}{2} - \sqrt{\frac{n}{2} \log \left( \frac{2M}{\delta} \right)}
    \end{align*}
    with probability $1 - \frac{\delta}{M}$ through the Hoeffding's inequality. Now, we need to set $n$ such that the right-hand side is greater than zero, which gives $n = \frac{2 \log (\frac{2M}{\delta})}{\lambda^4}$ and $h_0 = \frac{2 (M - 1) \log (\frac{2M}{\delta})}{\lambda^4}$.
    
    Finally, we set $\delta = \frac{1}{H}$ and we plug the expression into~\eqref{eq:bandit_upper_bound_decomposition}. Noting that, in the bad event occurring with probability less than $1 / H$ the right-hand side of~\eqref{eq:bandit_upper_bound_decomposition} is still less than $H$, we have $\EV [ \sum_{h = h_0}^H R_i(a^*) - R_i (\hat a) ] \leq 1$ from which we get
    \begin{equation*}
        R_H (\M_i, \hat \pi) = \cO \left( \frac{M \log (MH)}{\lambda^4} \right).
    \end{equation*}
\end{proof}


\end{document}

%% file: preamble.tex
\usepackage{url}            %
\usepackage{amsfonts}       %
\usepackage{dsfont}
\usepackage{nicefrac}       %
\usepackage{xcolor}         %
\usepackage{threeparttable}
\usepackage{colortbl}
\usepackage{multirow}
\definecolor{lightgray}{RGB}{235, 236, 240}

\usepackage{amsthm}
\usepackage{thmtools, thm-restate}
\declaretheorem[numberwithin=section]{thm}
\declaretheorem[sibling=thm]{theorem}
\declaretheorem[sibling=thm]{lemma}

\declaretheorem[]{assumption}
\declaretheorem[]{definition}

\usepackage{amsmath}
\usepackage{amssymb}
\usepackage{mathtools}
\usepackage{bm}

\newcommand{\M}{\mathcal{M}}

\newcommand{\cO}{\mathcal{O}}
\newcommand{\D}{\mathcal{D}}
\newcommand{\C}{\mathcal{C}}
\newcommand{\X}{\mathcal{X}}
\newcommand{\Prob}{\mathcal{P}}
\renewcommand{\S}{\mathcal{S}}
\newcommand{\A}{\mathcal{A}}

\newcommand{\Hist}{\mathcal{H}}

\newcommand{\R}{\mathcal{R}}
\newcommand{\F}{\mathcal{F}}

\DeclareMathOperator*{\EV}{\mathbb{E}}
\DeclareMathOperator*{\argmax}{arg\,max}
\DeclareMathOperator*{\argmin}{arg\,min}
\DeclareMathOperator{\KL}{KL}
\DeclareMathOperator{\TV}{TV}
\newcommand{\pibo}{\pi_{\text{BO}}}
\newcommand{\poly}{\mathsf{poly}}

\newcommand{\AssFam}{strong identifiability}

%% file: main.bbl
\begin{thebibliography}{51}
\providecommand{\natexlab}[1]{#1}
\providecommand{\url}[1]{\texttt{#1}}
\expandafter\ifx\csname urlstyle\endcsname\relax
  \providecommand{\doi}[1]{doi: #1}\else
  \providecommand{\doi}{doi: \begingroup \urlstyle{rm}\Url}\fi

\bibitem[Agarwal et~al.(2020)Agarwal, Kakade, and Yang]{agarwal2020model}
Agarwal, A., Kakade, S., and Yang, L.~F.
\newblock Model-based reinforcement learning with a generative model is minimax
  optimal.
\newblock In \emph{Conference on Learning Theory}, 2020.

\bibitem[Al~Marjani \& Proutiere(2021)Al~Marjani and Proutiere]{al2021adaptive}
Al~Marjani, A. and Proutiere, A.
\newblock Adaptive sampling for best policy identification in {M}arkov decision
  processes.
\newblock In \emph{International Conference on Machine Learning}, 2021.

\bibitem[Al~Marjani et~al.(2021)Al~Marjani, Garivier, and
  Proutiere]{al2021navigating}
Al~Marjani, A., Garivier, A., and Proutiere, A.
\newblock Navigating to the best policy in {M}arkov decision processes.
\newblock In \emph{Advances in Neural Information Processing Systems}, 2021.

\bibitem[Al-Marjani et~al.(2023)Al-Marjani, Tirinzoni, and
  Kaufmann]{al2023towards}
Al-Marjani, A., Tirinzoni, A., and Kaufmann, E.
\newblock Towards instance-optimality in online pac reinforcement learning.
\newblock \emph{arXiv preprint arXiv:2311.05638}, 2023.

\bibitem[Azar et~al.(2013)Azar, Munos, and Kappen]{gheshlaghi2013minimax}
Azar, M.~G., Munos, R., and Kappen, H.~J.
\newblock Minimax pac bounds on the sample complexity of reinforcement learning
  with a generative model.
\newblock \emph{Machine Learning}, 91:\penalty0 325--349, 2013.

\bibitem[Azar et~al.(2017)Azar, Osband, and Munos]{azar2017minimax}
Azar, M.~G., Osband, I., and Munos, R.
\newblock Minimax regret bounds for reinforcement learning.
\newblock In \emph{International Conference on Machine Learning}, 2017.

\bibitem[Bai et~al.(2019)Bai, Xie, Jiang, and Wang]{bai2019provably}
Bai, Y., Xie, T., Jiang, N., and Wang, Y.-X.
\newblock Provably efficient q-learning with low switching cost.
\newblock In \emph{Advances in Neural Information Processing Systems}, 2019.

\bibitem[Chatterji et~al.(2021)Chatterji, Pacchiano, Bartlett, and
  Jordan]{chatterji2021theory}
Chatterji, N., Pacchiano, A., Bartlett, P., and Jordan, M.
\newblock On the theory of reinforcement learning with once-per-episode
  feedback.
\newblock In \emph{Advances in Neural Information Processing Systems}, 2021.

\bibitem[Chen et~al.(2022)Chen, Hu, Jin, Li, and Wang]{chen2021understanding}
Chen, X., Hu, J., Jin, C., Li, L., and Wang, L.
\newblock Understanding domain randomization for sim-to-real transfer.
\newblock In \emph{International Conference on Learning Representations}, 2022.

\bibitem[Dann \& Brunskill(2015)Dann and Brunskill]{dann2015sample}
Dann, C. and Brunskill, E.
\newblock Sample complexity of episodic fixed-horizon reinforcement learning.
\newblock In \emph{Advances in Neural Information Processing Systems}, 2015.

\bibitem[Dann et~al.(2019)Dann, Li, Wei, and Brunskill]{dann2019policy}
Dann, C., Li, L., Wei, W., and Brunskill, E.
\newblock Policy certificates: Towards accountable reinforcement learning.
\newblock In \emph{International Conference on Machine Learning}, 2019.

\bibitem[Duan et~al.(2016)Duan, Schulman, Chen, Bartlett, Sutskever, and
  Abbeel]{duan2016rl}
Duan, Y., Schulman, J., Chen, X., Bartlett, P.~L., Sutskever, I., and Abbeel,
  P.
\newblock Rl\textsuperscript{2}: Fast reinforcement learning via slow
  reinforcement learning.
\newblock \emph{arXiv preprint arXiv:1611.02779}, 2016.

\bibitem[Efroni et~al.(2021)Efroni, Merlis, and
  Mannor]{efroni2021reinforcement}
Efroni, Y., Merlis, N., and Mannor, S.
\newblock Reinforcement learning with trajectory feedback.
\newblock In \emph{AAAI Conference on Artificial Intelligence}, 2021.

\bibitem[Fawzi et~al.(2022)Fawzi, Balog, Huang, Hubert, Romera-Paredes,
  Barekatain, Novikov, R~Ruiz, Schrittwieser, Swirszcz,
  et~al.]{fawzi2022discovering}
Fawzi, A., Balog, M., Huang, A., Hubert, T., Romera-Paredes, B., Barekatain,
  M., Novikov, A., R~Ruiz, F.~J., Schrittwieser, J., Swirszcz, G., et~al.
\newblock Discovering faster matrix multiplication algorithms with
  reinforcement learning.
\newblock \emph{Nature}, 610\penalty0 (7930):\penalty0 47--53, 2022.

\bibitem[Fiechter(1994)]{fiechter1994efficient}
Fiechter, C.-N.
\newblock Efficient reinforcement learning.
\newblock In \emph{Conference on Learning Theory}, 1994.

\bibitem[Finn et~al.(2017)Finn, Abbeel, and Levine]{finn2017model}
Finn, C., Abbeel, P., and Levine, S.
\newblock Model-agnostic meta-learning for fast adaptation of deep networks.
\newblock In \emph{International Conference on Machine Learning}, 2017.

\bibitem[Garivier \& Kaufmann(2016)Garivier and Kaufmann]{garivier2016optimal}
Garivier, A. and Kaufmann, E.
\newblock Optimal best arm identification with fixed confidence.
\newblock In \emph{Conference on Learning Theory}, 2016.

\bibitem[Ghavamzadeh et~al.(2015)Ghavamzadeh, Mannor, Pineau, and
  Tamar]{ghavamzadeh2015bayesian}
Ghavamzadeh, M., Mannor, S., Pineau, J., and Tamar, A.
\newblock Bayesian reinforcement learning: A survey.
\newblock \emph{Foundations and Trends{\textregistered} in Machine Learning},
  8\penalty0 (5-6):\penalty0 359--483, 2015.

\bibitem[Jaksch et~al.(2010)Jaksch, Ortner, and Auer]{jaksch2010near}
Jaksch, T., Ortner, R., and Auer, P.
\newblock Near-optimal regret bounds for reinforcement learning.
\newblock \emph{Journal of Machine Learning Research}, 11:\penalty0 1563--1600,
  2010.

\bibitem[Kaufmann et~al.(2016)Kaufmann, Capp{\'e}, and
  Garivier]{kaufmann2016complexity}
Kaufmann, E., Capp{\'e}, O., and Garivier, A.
\newblock On the complexity of best arm identification in multi-armed bandit
  models.
\newblock \emph{Journal of Machine Learning Research}, 17:\penalty0 1--42,
  2016.

\bibitem[Kaufmann et~al.(2021)Kaufmann, M{\'e}nard, Domingues, Jonsson,
  Leurent, and Valko]{kaufmann2021adaptive}
Kaufmann, E., M{\'e}nard, P., Domingues, O.~D., Jonsson, A., Leurent, E., and
  Valko, M.
\newblock Adaptive reward-free exploration.
\newblock In \emph{Algorithmic Learning Theory}, 2021.

\bibitem[Kaufmann et~al.(2023)Kaufmann, Bauersfeld, Loquercio, M{\"u}ller,
  Koltun, and Scaramuzza]{kaufmann2023champion}
Kaufmann, E., Bauersfeld, L., Loquercio, A., M{\"u}ller, M., Koltun, V., and
  Scaramuzza, D.
\newblock Champion-level drone racing using deep reinforcement learning.
\newblock \emph{Nature}, 620\penalty0 (7976):\penalty0 982--987, 2023.

\bibitem[Kausik et~al.(2023)Kausik, Tan, and Tewari]{kausik2023learning}
Kausik, C., Tan, K., and Tewari, A.
\newblock Learning mixtures of markov chains and mdps.
\newblock In \emph{International Conference on Machine Learning}, 2023.

\bibitem[Kumar et~al.(2024)Kumar, Derman, Geist, Levy, and
  Mannor]{kumar2024policy}
Kumar, N., Derman, E., Geist, M., Levy, K.~Y., and Mannor, S.
\newblock Policy gradient for rectangular robust markov decision processes.
\newblock In \emph{Advances in Neural Information Processing Systems}, 2024.

\bibitem[Kwon et~al.(2021{\natexlab{a}})Kwon, Efroni, Caramanis, and
  Mannor]{kwon2021reinforcement}
Kwon, J., Efroni, Y., Caramanis, C., and Mannor, S.
\newblock Reinforcement learning in reward-mixing mdps.
\newblock In \emph{Advances in Neural Information Processing Systems},
  2021{\natexlab{a}}.

\bibitem[Kwon et~al.(2021{\natexlab{b}})Kwon, Efroni, Caramanis, and
  Mannor]{kwon2021rl}
Kwon, J., Efroni, Y., Caramanis, C., and Mannor, S.
\newblock Rl for latent mdps: Regret guarantees and a lower bound.
\newblock In \emph{Advances in Neural Information Processing Systems},
  2021{\natexlab{b}}.

\bibitem[Kwon et~al.(2023{\natexlab{a}})Kwon, Efroni, Caramanis, and
  Mannor]{kwon2023reward}
Kwon, J., Efroni, Y., Caramanis, C., and Mannor, S.
\newblock Reward-mixing mdps with few latent contexts are learnable.
\newblock In \emph{International Conference on Machine Learning},
  2023{\natexlab{a}}.

\bibitem[Kwon et~al.(2023{\natexlab{b}})Kwon, Efroni, Mannor, and
  Caramanis]{kwon2023prospective}
Kwon, J., Efroni, Y., Mannor, S., and Caramanis, C.
\newblock Prospective side information for latent mdps.
\newblock \emph{arXiv preprint arXiv:2310.07596}, 2023{\natexlab{b}}.

\bibitem[Lattimore \& Szepesv{\'a}ri(2020)Lattimore and
  Szepesv{\'a}ri]{lattimore2020bandit}
Lattimore, T. and Szepesv{\'a}ri, C.
\newblock \emph{Bandit algorithms}.
\newblock Cambridge University Press, 2020.

\bibitem[M{\'e}nard et~al.(2021)M{\'e}nard, Domingues, Jonsson, Kaufmann,
  Leurent, and Valko]{menard2021fast}
M{\'e}nard, P., Domingues, O.~D., Jonsson, A., Kaufmann, E., Leurent, E., and
  Valko, M.
\newblock Fast active learning for pure exploration in reinforcement learning.
\newblock In \emph{International Conference on Machine Learning}, 2021.

\bibitem[Mnih et~al.(2015)Mnih, Kavukcuoglu, Silver, Rusu, Veness, Bellemare,
  Graves, Riedmiller, Fidjeland, Ostrovski, et~al.]{mnih2015human}
Mnih, V., Kavukcuoglu, K., Silver, D., Rusu, A.~A., Veness, J., Bellemare,
  M.~G., Graves, A., Riedmiller, M., Fidjeland, A.~K., Ostrovski, G., et~al.
\newblock Human-level control through deep reinforcement learning.
\newblock \emph{Nature}, 518\penalty0 (7540):\penalty0 529--533, 2015.

\bibitem[Osband \& Van~Roy(2016)Osband and Van~Roy]{osband2016lower}
Osband, I. and Van~Roy, B.
\newblock On lower bounds for regret in reinforcement learning.
\newblock \emph{arXiv preprint arXiv:1608.02732}, 2016.

\bibitem[Puterman(2014)]{puterman2014markov}
Puterman, M.~L.
\newblock \emph{Markov decision processes: discrete stochastic dynamic
  programming}.
\newblock John Wiley \& Sons, 2014.

\bibitem[Rakelly et~al.(2019)Rakelly, Zhou, Finn, Levine, and
  Quillen]{rakelly2019efficient}
Rakelly, K., Zhou, A., Finn, C., Levine, S., and Quillen, D.
\newblock Efficient off-policy meta-reinforcement learning via probabilistic
  context variables.
\newblock In \emph{International Conference on Machine Learning}, 2019.

\bibitem[Rimon et~al.(2022)Rimon, Tamar, and Adler]{rimon2022meta}
Rimon, Z., Tamar, A., and Adler, G.
\newblock Meta reinforcement learning with finite training tasks-a density
  estimation approach.
\newblock In \emph{Advances in Neural Information Processing Systems}, 2022.

\bibitem[Schmidhuber(1987)]{schmidhuber1987evolutionary}
Schmidhuber, J.
\newblock \emph{Evolutionary principles in self-referential learning, or on
  learning how to learn: the meta-meta-... hook}.
\newblock PhD thesis, Technische Universit{\"a}t M{\"u}nchen, 1987.

\bibitem[Schulman et~al.(2015)Schulman, Levine, Abbeel, Jordan, and
  Moritz]{schulman2015trust}
Schulman, J., Levine, S., Abbeel, P., Jordan, M., and Moritz, P.
\newblock Trust region policy optimization.
\newblock In \emph{International conference on machine learning}, 2015.

\bibitem[Simchowitz et~al.(2021)Simchowitz, Tosh, Krishnamurthy, Hsu, Lykouris,
  Dudik, and Schapire]{simchowitz2021bayesian}
Simchowitz, M., Tosh, C., Krishnamurthy, A., Hsu, D.~J., Lykouris, T., Dudik,
  M., and Schapire, R.~E.
\newblock Bayesian decision-making under misspecified priors with applications
  to meta-learning.
\newblock In \emph{Advances in Neural Information Processing Systems}, 2021.

\bibitem[Stiennon et~al.(2020)Stiennon, Ouyang, Wu, Ziegler, Lowe, Voss,
  Radford, Amodei, and Christiano]{stiennon2020learning}
Stiennon, N., Ouyang, L., Wu, J., Ziegler, D., Lowe, R., Voss, C., Radford, A.,
  Amodei, D., and Christiano, P.~F.
\newblock Learning to summarize with human feedback.
\newblock In \emph{Advances in Neural Information Processing Systems}, 2020.

\bibitem[Sutton \& Barto(2018)Sutton and Barto]{sutton2018reinforcement}
Sutton, R.~S. and Barto, A.~G.
\newblock \emph{Reinforcement learning: An introduction}.
\newblock MIT press, 2018.

\bibitem[Tamar et~al.(2022)Tamar, Soudry, and
  Zisselman]{tamar2022regularization}
Tamar, A., Soudry, D., and Zisselman, E.
\newblock Regularization guarantees generalization in bayesian reinforcement
  learning through algorithmic stability.
\newblock In \emph{AAAI Conference on Artificial Intelligence}, 2022.

\bibitem[Tirinzoni et~al.(2022)Tirinzoni, Al~Marjani, and
  Kaufmann]{tirinzoni2022near}
Tirinzoni, A., Al~Marjani, A., and Kaufmann, E.
\newblock Near instance-optimal pac reinforcement learning for deterministic
  mdps.
\newblock In \emph{Advances in Neural Information Processing Systems}, 2022.

\bibitem[Tirinzoni et~al.(2023)Tirinzoni, Al-Marjani, and
  Kaufmann]{tirinzoni2023optimistic}
Tirinzoni, A., Al-Marjani, A., and Kaufmann, E.
\newblock Optimistic pac reinforcement learning: the instance-dependent view.
\newblock In \emph{Algorithmic Learning Theory}, 2023.

\bibitem[Wagenmaker et~al.(2022)Wagenmaker, Simchowitz, and
  Jamieson]{wagenmaker2022beyond}
Wagenmaker, A.~J., Simchowitz, M., and Jamieson, K.
\newblock Beyond no regret: Instance-dependent pac reinforcement learning.
\newblock In \emph{Conference on Learning Theory}, 2022.

\bibitem[Wiesemann et~al.(2013)Wiesemann, Kuhn, and
  Rustem]{wiesemann2013robust}
Wiesemann, W., Kuhn, D., and Rustem, B.
\newblock Robust markov decision processes.
\newblock \emph{Mathematics of Operations Research}, 2013.

\bibitem[Xu et~al.(2023)Xu, Hu, Doshi, Rovinsky, Kumar, Gupta, and
  Levine]{xu2023dexterous}
Xu, K., Hu, Z., Doshi, R., Rovinsky, A., Kumar, V., Gupta, A., and Levine, S.
\newblock Dexterous manipulation from images: Autonomous real-world rl via
  substep guidance.
\newblock In \emph{IEEE International Conference on Robotics and Automation},
  2023.

\bibitem[Ye et~al.(2023)Ye, Chen, Wang, and Du]{ye2023power}
Ye, H., Chen, X., Wang, L., and Du, S.~S.
\newblock On the power of pre-training for generalization in {RL}: Provable
  benefits and hardness.
\newblock In \emph{International Conference on Machine Learning}, 2023.

\bibitem[Zhao et~al.(2022)Zhao, Abbeel, and James]{zhao2022effectiveness}
Zhao, M., Abbeel, P., and James, S.
\newblock On the effectiveness of fine-tuning versus meta-reinforcement
  learning.
\newblock In \emph{Advances in Neural Information Processing Systems}, 2022.

\bibitem[Zintgraf et~al.(2019)Zintgraf, Shiarlis, Igl, Schulze, Gal, Hofmann,
  and Whiteson]{zintgraf2019varibad}
Zintgraf, L., Shiarlis, K., Igl, M., Schulze, S., Gal, Y., Hofmann, K., and
  Whiteson, S.
\newblock Varibad: A very good method for bayes-adaptive deep rl via
  meta-learning.
\newblock In \emph{International Conference on Learning Representations}, 2019.

\bibitem[Zintgraf et~al.(2021)Zintgraf, Schulze, Lu, Feng, Igl, Shiarlis, Gal,
  Hofmann, and Whiteson]{zintgraf2021varibad}
Zintgraf, L., Schulze, S., Lu, C., Feng, L., Igl, M., Shiarlis, K., Gal, Y.,
  Hofmann, K., and Whiteson, S.
\newblock Varibad: Variational bayes-adaptive deep rl via meta-learning.
\newblock \emph{The Journal of Machine Learning Research}, 22\penalty0
  (1):\penalty0 13198--13236, 2021.

\bibitem[Zisselman et~al.(2023)Zisselman, Lavie, Soudry, and
  Tamar]{zisselman2023explore}
Zisselman, E., Lavie, I., Soudry, D., and Tamar, A.
\newblock Explore to generalize in zero-shot rl.
\newblock In \emph{Advances in Neural Information Processing Systems}, 2023.

\end{thebibliography}
